\documentclass[acmlarge,11pt]{acmart}

\copyrightyear{2019} 
\acmYear{2019} 
\setcopyright{acmcopyright}
\acmConference[FAT* '19]{FAT* '19: Conference on Fairness, Accountability, and Transparency}{January 29--31, 2019}{Atlanta, GA, USA}
\acmBooktitle{FAT* '19: Conference on Fairness, Accountability, and Transparency (FAT* '19), January 29--31, 2019, Atlanta, GA, USA}
\acmPrice{15.00}
\acmDOI{10.1145/3287560.3287566}
\acmISBN{978-1-4503-6125-5/19/01}

\renewcommand\footnotetextcopyrightpermission[1]{} %

\usepackage{subdepth}
\usepackage{graphicx}
\usepackage{amsfonts,bbm,bm,courier}
\usepackage{enumitem}
\usepackage{relsize}
\usepackage{setspace}
\usepackage{xcolor}

\usepackage{placeins}

\PassOptionsToPackage{hyphens}{url}
\usepackage{hyperref}
\hypersetup{%
  colorlinks=true,
  urlcolor=blue,%
  linkcolor=blue,%
  citecolor=blue,%
  breaklinks=true,%
  linkbordercolor=blue,
  pdfborderstyle={/S/U/W 1}
}

\usepackage{booktabs}
\usepackage{array}
\usepackage{tabularx}
\usepackage{colortbl}
\usepackage{multirow}
\usepackage{hhline}
\newcommand{\cell}[2]{\setlength{\tabcolsep}{0pt}\begin{tabular}{#1}#2 \end{tabular}}

\usepackage{eqnarray}
\usepackage{mathtools}

\newtheorem{remark}{Remark}
\newtheorem{example}{Example}
\newtheorem{definition}{Definition}
\setitemize{label=\raisebox{0.25ex}{\tiny$\bullet$}}
\setitemize{leftmargin=1em,topsep=0.5em,parsep=0.5em, label=\raisebox{0.25ex}{\tiny$\bullet$}}
\setlist[enumerate]{leftmargin=*, label= {\arabic*.}, itemsep=0.5em}
\newlist{thmlist}{enumerate}{1}
\setlist[thmlist]{leftmargin=*,label=\raisebox{0.25ex}{\tiny$\bullet$}, topsep=0.2em,itemsep=2pt}

\newcommand{\textds}[1]{{\footnotesize\texttt{#1}}}
\newcommand{\textfn}[1]{{\textit{#1}}}
\newcommand{\texttrue}{{\footnotesize\texttt{TRUE}}}
\newcommand{\textfalse}{{\footnotesize\texttt{FALSE}}}

\newcommand{\sign}[1]{\textnormal{sign}(#1)}
\newcommand{\yhat}[1]{\hat{y}_{#1}}
\newcommand{\w}{\bm{w}}
\newcommand{\xb}{\bm{x}}

\newcommand{\X}{\mathcal{X}}

\renewcommand{\a}{\bm{a}}
\renewcommand{\b}{\bm{b}}
\newcommand{\A}{A}

\newcommand{\E}{\mathbb{E}}
\newcommand{\R}{\mathbb{R}}
\newcommand{\B}{\{0,1\}}

\newcommand{\st}{\textnormal{s.t.}}
\newcommand{\miprange}[3]{{#1}={#2},\ldots,{#3}}

\DeclareMathOperator*{\argmin}{argmin}

\newcommand{\vnorm}[1]{\left\|#1\right\|}

\newcommand{\twonorm}[1]{||#1||_2}
\newcommand{\indic}[1]{1\left[#1\right]}

\newcommand{\prob}[2]{\textnormal{Pr}_{#2}\left(#1\right)}
\newcommand{\dotprod}[2]{\langle{#1},{#2} \rangle}

\newcommand{\Aset}{A}
\newcommand{\cost}[2]{\textnormal{cost}({#1};{#2})}

\newcommand{\J}{J}
\newcommand{\Ja}{\J_A}
\newcommand{\Jn}{\J_N}

\newcommand{\wa}{\w_A}
\newcommand{\wn}{\w_N}
\newcommand{\xa}{\xb_A}

\newcommand{\hpos}{H^+}
\newcommand{\hneg}{H^-}
\newcommand{\dpos}{D^+}
\newcommand{\dneg}{D^-}%

\newcommand{\cx}[0]{c_{\xb}}
\newcommand{\ucmax}[0]{\gamma^{\max{}}_A}
\newcommand{\ucpos}[0]{\gamma^{+}_A}
\newcommand{\ucneg}[0]{\gamma^{-}_A}

\newcommand{\expmincostneg}[1]{\overline{\textnormal{cost}}_{\hneg{}}({#1})}

\usepackage{algorithm,algpseudocode}

\algnewcommand\algorithmicinput{\textbf{Input}}

\algnewcommand\algorithmicinitialize{\textbf{Initialize}}
\algnewcommand\algorithmicbigstep{\textbf{Step}}
\algnewcommand\INPUT{\item[\algorithmicinput]}
\algnewcommand\INITIALIZE{\item[\algorithmicinitialize]}
\algnewcommand{\STEP}[1]{\item[\algorithmicbigstep]{\textbf{#1}}}
\algnewcommand{\InputExplanation}[2][.6\linewidth]{\leavevmode\hfill\makebox[#1][r]{~{\footnotesize{#2}}}}
\algnewcommand{\InitializationExplanation}[2][.6\linewidth]{\leavevmode\hfill\makebox[#1][r]{~{\footnotesize{#2}}}}
\algnewcommand{\alginput}[2]{\Statex{#1}\InputExplanation{#2}}
\algnewcommand{\StateComment}[2]{\State{#1}\InputExplanation{#2}}
\algnewcommand{\alginitialize}[2]{\Statex{#1}\InitializationExplanation{#2}}
\algrenewcommand\algorithmiccomment[2][]{#1\hfill\textit{\scriptsize{#2}}}

\begin{document}
\title{Actionable Recourse in Linear Classification}

\author{Berk Ustun}
\orcid{5188-3155}
\affiliation{%
  \institution{Harvard University}
}
\email{berk@seas.harvard.edu}

\author{Alexander Spangher}
\orcid{6655-9593}
\affiliation{%
  \institution{University of Southern California}
}
\email{spangher@usc.edu}

\author{Yang Liu}
\affiliation{%
  \institution{University of California, Santa Cruz}
}
\email{yangliu@ucsc.edu}

\renewcommand{\shortauthors}{Ustun et al.}

\begin{abstract}

Machine learning models are increasingly used to automate decisions that affect humans --- deciding who should receive a loan, a job interview, or a social service. In such applications, a person should have the ability to change the decision of a model. When a person is denied a loan by a credit score, for example, they should be able to alter its input variables in a way that guarantees approval. Otherwise, they will be denied the loan as long as the model is deployed. More importantly, they will lack the ability to influence a decision that affects their livelihood.

In this paper, we frame these issues in terms of \emph{recourse}, which we define as the ability of a person to change the decision of a model by altering \emph{actionable} input variables (e.g., income vs. age or marital status). We present integer programming tools to ensure recourse in linear classification problems without interfering in model development. We demonstrate how our tools can inform stakeholders through experiments on credit scoring problems. Our results show that recourse can be significantly affected by standard practices in model development, and motivate the need to evaluate recourse in practice.

\end{abstract}


\maketitle
\section{Introduction}
\label{Sec::Introduction}
\normalsize

In the context of machine learning, we define \emph{recourse} as the ability of a person to obtain a desired outcome from a fixed model. Consider a classifier used for loan approval. If the model provides recourse to someone who is denied a loan, then this person can alter its input variables in a way that guarantees approval. Otherwise, this person will be denied the loan so long as the model is deployed, and will lack the ability to influence a decision that affects their livelihood.

Recourse is not formally studied in machine learning. In this paper, we argue that it should be. A model should provide recourse to its decision subjects in applications such as lending~\citep{siddiqi2012credit}, hiring ~\citep{ajunwa2016hiring,bogen2018hiring}, insurance~\citep{wsj2019lifeinsurance}, or the allocation of public services ~\citep{chouldechova2018case,shroff2017predictive}. Seeing how the lack of autonomy is perceived as a source of injustice in algorithmic decision-making ~\cite{binns2018s,cathy2016weapons,crawford2014big}, recourse should be considered whenever humans are subject to the predictions of a machine learning model.

The lack of recourse is often mentioned in calls for increased transparency and explainability in algorithmic decision-making~\citep[see e.g.,][]{citron2014scored,wachter2017counterfactual,doshi2017accountability}. Yet, transparency and explainability do not provide meaningful protection with regards to recourse. In fact, even simple transparent models such as linear classifiers may not provide recourse to all of their decision subjects due to widespread practices in machine learning. These include:
\begin{itemize}
		
\item \emph{Choice of Features}: A model could use features that are immutable (e.g., \textfn{age} $\geq 50$), conditionally immutable (e.g., \textfn{has\_phd}, which can only change from $\textfalse{} \to \texttrue{}$), or should not be considered actionable (e.g., \textfn{married}).
				
\item \emph{Out-of-Sample Deployment}: The ability of a model to provide recourse may depend on a feature that is missing, immutable, or adversely distributed in the deployment population.

\item \emph{Choice of Operating Point}: A probabilistic classifier may provide recourse at a given threshold (e.g., $\yhat{i} = 1$ if predicted risk of default $\geq 50\%$) but fail to provide recourse at a more stringent threshold (e.g., $\yhat{i} = 1$ if predicted risk of default $\geq 80\%$).

\item \emph{Drastic Changes}: A model could provide recourse to all individuals but require some individuals to make drastic changes (e.g., increase \textfn{income} from $\$50\textrm{K} \to \$1\textrm{M}$).

\end{itemize}
Considering these failure modes, an ideal attempt to protect recourse should evaluate both the \emph{feasibility} and \emph{difficulty} of recourse for individuals in a model's deployment population (i.e., its \emph{target population}).

In this paper, we present tools to evaluate recourse for linear classification models, such as logistic regression models, linear SVMs, and linearizable rule-based models (e.g., rule sets, decision lists). Our tools are designed to ensure recourse without interfering in model development. To this end, they aim to answer questions such as:
\begin{itemize}\itshape
    \item Will a model provide recourse to all its decision subjects?
    \item How does the difficulty of recourse vary in a population of interest?
    \item What can a person change to obtain a desired prediction from a particular model?
\end{itemize}
We answer these questions by solving a hard discrete optimization problem. This problem searches over changes that a specific person can make to ``flip" the prediction of a fixed linear classifier. It includes discrete constraints so that it will only consider \emph{actionable} changes --- i.e., changes that do not alter immutable features and that do not alter mutable features in an infeasible way (e.g., \textfn{n\_credit\_cards} from $5\to 0.5$, or \textfn{has\_phd} from $\texttrue{} \to \textfalse{}$). We develop an efficient routine to solve this optimization problem, by expressing it as an \emph{integer program} (IP) and handing it to an IP solver (e.g., CPLEX or CBC). We use our routine to create the following tools:

\begin{enumerate}

\item A procedure to evaluate the feasibility and difficulty of recourse for a linear classifier over its target population. Given a classifier and a sample of feature vectors from a target population, our procedure estimates the feasibility and difficulty of recourse in the population by solving the optimization problem for each point that receives an undesirable prediction. This procedure provides a way to check recourse in model development, procurement, or impact assessment~\citep[see e.g,][]{reisman2018aia,senate2019aaa}. 

\item A method to generate a list of actionable changes for a person to obtain a desired outcome from a linear classifier. We refer to this list as a \emph{flipset} and present an example in Figure~\ref{Fig::ExampleFlipset}. In the United States, the Equal Opportunity Credit Act~\cite{congress2003facta} requires that any person who is denied credit is sent an \emph{adverse action notice} explaining ``the principal reason for the denial." It is well-known that adverse action notices may not provide actionable information~\citep[see e.g.,][for a critique]{taylor1980meeting}. By including a flipset in an adverse action notice, a person would know a set of exact changes to be approved in the future.

\end{enumerate}

\begin{figure}[htbp]
\centering
{
\scriptsize
\setlength{\tabcolsep}{3pt}
\renewcommand{\arraystretch}{1.025}
\begin{tabular}{lccc}
\toprule
\textsc{Features to Change} & \textsc{Current Values} & & \textsc{Required Values} \\ 
\toprule
\textfn{n\_credit\_cards} & 5 & $\longrightarrow$ & 3\\ \midrule
\textfn{current\_debt} & \$3,250 & $\longrightarrow$ & \$1,000 \\ \midrule
\textfn{has\_savings\_account} & \textfalse{}& $\longrightarrow$ & \texttrue{} \\
\textfn{has\_retirement\_account} & \textfalse{} & $\longrightarrow$ & \texttrue{} \\ 
\bottomrule
\end{tabular}
}
\caption{Hypothetical flipset for a person who is denied credit by a classifier. Each row (\emph{item}) describes how a subset of features that the person can change to ``flip" the prediction of the model from $\yhat{} =-1 \to +1$.}
\label{Fig::ExampleFlipset}
\end{figure}

\subsubsection*{Related Work}

Recourse is broadly related to a number of different topics in machine learning. These include: \emph{inverse classification}, which aims to determine how the inputs to a model can be manipulated to obtain a desired outcome~\citep{aggarwal2010inverse,chang2012reverse}; \emph{strategic classification}, which considers how to build classifiers that are robust to malicious manipulation ~\citep{hardt2016strategic,dong2018strategic,milli2019strategic,hu2019strategic,cowgill2019economics}; \emph{adversarial perturbations}, which studies the robustness of predictions with respect to small changes in inputs~\citep{fawzi2018analysis}; and \emph{anchors}, which are subsets of features that fix the prediction of a model~\citep{ribeiro2018anchors, hara2018maximally}. 

The study of recourse involves determining the existence and difficulty of \emph{actions} to obtain a desired prediction from a fixed machine learning model. Such actions do not reflect the principle reasons for the prediction (c.f., explainability), and are not designed to reveal the operational process of the model (c.f., transparency). Nevertheless, simple transparent models~\cite[e.g.,][]{ustun2016slim, ustun2016kdd, sokolovska2018provable, malioutov2013exact, angelino2017learning} have a benefit in that they allow users to check the feasibility of recourse without extensive training or electronic assistance.

Methods to explain the predictions of a machine learning model \citep[see e.g.,][]{poulin2006visual, lim2009assessing, biran2014justification, ribeiro2016should} do not produce useful information with regards to recourse. This is because: (i) their explanations do not reveal actionable changes that produce a desired prediction; and (ii) if a method fails to find an actionable change, an actionable change may still exist. 
\footnote{ 
For example, consider the method of~\citet{wachter2017counterfactual} to produce counterfactual explanations from a black-box classifier. This method does not produce useful information about recourse because: (a) it does not constrain changes to be actionable; (b) it assumes that a feasible changes must be observed in the training data (i.e., a feasible action is defined as $\a \in \{\xb-\xb'\},$ where $\xb, \xb'$ are points in the training data). In practice, this method could output an explanation stating that a person can flip their prediction by changing an immutable attribute, due to (a). In this case, one cannot claim that the model fails to provide recourse, because there may exist a way to flip the prediction that is not observed in the training data, due to (b).
}
Note that (ii) is a key requirement to verify the feasibility of recourse.%

In contrast, our tools overcome limitations of methods to generate counterfactual explanations for linear classification problems~\cite[e.g.,][]{martens2014explaining, wachter2017counterfactual}. In particular, they can be used to: (i) produce counterfactual explanations that obey discrete constraints; (ii) prove that specific kinds of counterfactual explanations do not exist (e.g., actionable explanations); (iii) enumerate all counterfactual explanations for a given prediction; and (iv) choose between competing counterfactual explanations using a custom cost function (c.f., a Euclidean distance metric).

\subsubsection*{Software and Workshop Paper}

This paper extends work that was first presented at FAT/ML 2018~\citep{recourse2018fatml}. We provide an open-source implementation of our tools at \url{http://github.com/ustunb/actionable-recourse}. 


\clearpage
\section{Problem Statement}
\label{Sec::RecourseProblem}

In this section, we define the optimization problem that we solve to evaluate recourse, and present guarantees on the feasibility and cost of recourse. We include proofs for all results in Appendix \ref{Appendix::Proofs}. 

\subsection{Optimization Framework}

We consider a standard classification problem where each person is characterized by a \emph{feature vector} $\xb = [1, x_{1}\ldots x_{d}] \subseteq \X_0 \cup \ldots \cup \X_d = \X \subseteq \R^{d+1}$ and a binary \emph{label} $y \in \{-1,+1\}$. We assume that we are given a linear classifier $f(\xb) = \sign{\dotprod{\w}{\xb}}$ where $\w = [w_0, w_1, \ldots, w_d] \subseteq \R^{d+1}$ is a vector of coefficients and $w_0$ is the intercept. We denote the \emph{desired outcome} as $\yhat{} = 1$, and assume that $\yhat{} = \indic{\dotprod{\w}{\xb} \geq 0}$. 

Given a person who is assigned an undesirable outcome $f(\xb) = -1$, we aim to find an \emph{action} $\a$ such that $f(\xb + \a) = +1$ by solving an optimization problem of the form,
\begin{align}
\begin{split}
\min\quad& \cost{\a}{\xb} \\ 
\st\quad & f(\xb + \a) = +1,\\ \label{Eq::RecourseProblem}
& \a \in \A(\xb).
\end{split}
\end{align}
Here:
\begin{itemize}

\item $\A(\xb)$ is a set of feasible actions from $\xb$. Each \emph{action} is a vector $\a = [0, a_{1}, \ldots, a_{d}]$ where $a_{j} \in \Aset_{j}(x_j) \subseteq \{a_{j} \in \R ~|~ a_{j} + x_{j} \in \X_j\}$.  We say that a feature $j$ is \emph{immutable} if $A_j(\xb) = \{0\}$. We say that feature $j$ is \emph{conditionally immutable} if $\A_j(\xb) = \{0\}$ for some $\xb \in \X.$

\item $\cost{\,\cdot\,}{\xb}\,:\, \A(\xb) \rightarrow \R_+$ is a \emph{cost function} to choose between feasible actions, or to measure quantities of interest in a recourse audit (see Section \ref{Sec::CostFunction}). We assume that cost functions satisfy the following properties: (i) $\cost{\bm{0}}{\xb} = 0$ (no action $\Leftrightarrow$ no cost); (ii) $\cost{\a}{\xb} \leq \cost{\a + \epsilon \bm{1}_j}{\xb}$ (larger actions $\Leftrightarrow$ higher cost). 

\end{itemize}
Solving \eqref{Eq::RecourseProblem} allows us to make one of the following claims related to recourse:
\begin{itemize}

\item If \eqref{Eq::RecourseProblem} is \emph{feasible}, then its optimal solution $\a{}^*$ is the minimal-cost action to flip the prediction of $\xb$. 

\item If \eqref{Eq::RecourseProblem} is \emph{infeasible}, then no action can attain a desired outcome from $\xb$. Thus, we have certified that the model $f$ does not provide actionable recourse for a person with features $\xb$.
    
\end{itemize} 

\subsubsection*{Assumptions and Notation}

Given a linear classifier of the form $f(\xb) = \sign{\dotprod{\w}{\xb}}$, we denote the  coefficients of actionable and immutable features as $\wa$ and $\wn$, respectively. We denote the indices of all features as $\J = \{1,\ldots,d\}$, of immutable features as $\Jn(\xb) = \{j \in \J ~|~ \A_j(\xb) = 0 \}$, and of actionable features as $\Ja(\xb) = \{ j \in \J ~|~ |A_j(\xb)| > 1 \}$. We write $\Ja$ and $\Jn$ when the dependence of these sets on $\xb$ is clear from context. We assume that features are bounded so that $\vnorm{\xb} \leq B$ for all $\xb \in \X$ where $B$ is a sufficiently large constant. We define the following subspaces of the $\X$ based on the values of $y$ and $f(\xb)$:
\begin{align*}
    \dneg &=\{\xb \in \X: y = -1\} &  \dpos &=\{\xb \in \X: y = +1\}\\
    \hneg &= \{\xb \in \X: f(\xb) = -1\} &  \hpos &= \{\xb \in \X: f(\xb) = +1\}.
\end{align*}


\subsection{Feasibility Guarantees}
\label{Sec::FeasibilityGuarantees}

We start with a simple sufficient condition for a linear classifier to provide a universal recourse guarantee (i.e., to provide recourse to all individuals in any target population).
\begin{remark}
\label{Rem::FeasibilitySufficient}
A linear classifier provides recourse to all individuals if it only uses actionable features and does not predict a single class.
\end{remark}
Remark~\ref{Rem::FeasibilitySufficient} is useful in settings where models must provide recourse. For instance, the result could be used to design screening questions for an algorithmic impact assessment (e.g., ``can a person affected by this model alter all of its features, regardless of their current values?"). 

The converse of Remark~\ref{Rem::FeasibilitySufficient} is also true -- a classifier denies recourse to all individuals if it uses immutable features exclusively or it predicts a single class consistently. In what follows, we consider models that deny recourse in non-trivial ways. The following remarks apply to linear classifiers with non-zero coefficients $\w \neq \bm{0}$ that predict both classes in the target population.
\begin{remark}
\label{Remark::UnboundedFeasibility}
If all features are unbounded, then a linear classifier with at least one actionable feature provides recourse to all individuals.
\end{remark}
\begin{remark}
\label{Remark::BoundedInfeasibility}
If all features are bounded, then a linear classifier with at least one immutable feature may deny recourse to some individuals.
\end{remark}
Remarks~\ref{Remark::UnboundedFeasibility} and~\ref{Remark::BoundedInfeasibility} show how the feasibility of recourse depends on the bound of actionable features. To make meaningful claims about the feasibility of recourse, these bounds must be set judiciously. In general, we only need to specify bounds for some kinds of features since many features are bounded by definition (e.g., features that are binary, ordinal, or categorical). As such, the validity of a feasibility claim only depends on the bounds for actionable features, such as \textfn{income} or \textfn{n\_credit\_cards}.  In practice, we would set loose bounds for such features to avoid claiming infeasibility due to overly restrictive bounds. This allows a classifier to provide recourse superficially by demanding drastic changes. However, such cases will be easy to spot in an audit as they will incur large costs (assuming that we use an informative cost function such the one in Section \ref{Sec::CostFunction}).

Recourse is not guaranteed when a classifier uses features that are immutable or conditionally immutable (e.g., \textfn{age} or \textfn{has\_phd}). As shown in Example \ref{Ex::BayesImmutable}, a classifier with only one immutable feature could achieve perfect predictive accuracy without providing a universal recourse guarantee. In practice, it may be desirable to include such features in a model because they improve predictive performance or provide robustness to manipulation.
\begin{example}
\label{Ex::BayesImmutable}
Consider training a linear classifier using $n$ examples $(\xb_i, y_i)_{i=1}^{n}$ where $\xb_i \in \{0,1\}^d$ and $y_i \in \{-1,+1\}$ where each label is drawn from the distribution $$\prob{y = +1 |\xb}{} = \frac{1}{1 + \exp(\alpha -\sum_{j=1}^d{x_j})}.$$ In this case, the Bayes optimal classifier is $f(\xb) = \textnormal{sgn}(\sum_{j=1}^d x_{j} - \alpha).$ If $\alpha > d - 1$, then $f$ will deny recourse to any person with $x_j = 0$ for an immutable feature $j \in \Jn.$
\end{example}
%

\clearpage
\subsection{Cost Guarantees}
\label{Sec::CostGuarantees}

In Theorem \ref{Thm::ExpectedCost}, we present a bound on the expected cost of recourse.
\begin{definition}
The \emph{expected cost of recourse} of a classifier $f: \X \to \{-1,+1\},$ is defined as: $$\expmincostneg{f} = \E_{\hneg{}}[\cost{\a^*}{\xb}],$$ where $\a^*$ is an optimal solution to the optimization problem in~\eqref{Eq::RecourseProblem}.
\end{definition}
Our guarantee is expressed in terms of cost function with the form $\cost{\a}{\xb} = \cx{} \cdot \vnorm{\a},$ where $\cx \in (0, +\infty)$ is a positive scaling constant for actions from $\xb \in \X$, and $\X$ is a closed convex set.
\begin{theorem}
\label{Thm::ExpectedCost}
The expected cost of recourse of a linear classifier over a target population obeys:
\begin{align*}
\expmincostneg{f} \leq p^+ \ucpos{}+ p^- \ucneg{} + 2 \ucmax{} R_A(f),
\end{align*}
where:
\begin{thmlist}

    \item $p^{+} = \prob{y=+1}{\hneg}$ is the false omission rate of $f$;
    
    \item $p^{-} = \prob{y=-1}{\hneg}$ is the negative predictive value of $f$;
    
    \item $\ucpos{} = \E_{\hneg{}\cap \dpos} [\cx\cdot \frac{\wa^{\top}\xa}{\twonorm{\wa}^2}]$ is the expected unit cost of actionable changes for false negatives;
    
    \item $\ucneg{} = \E_{\hneg{}\cap \dneg} [\cx\cdot \frac{-\wa^{\top}\xa}{\twonorm{\wa}^2}]$ is the expected unit cost of actionable changes for true negatives;
    
    \item $\ucmax{} = \max_{\xb \in \hneg{}}\bigl|\cx\cdot \frac{\wa^\top\xa}{\twonorm{\wa}^2} \bigr|$ is the maximum unit cost of actionable changes for negative predictions;
    
    \item $R_A(f) = p^+ \cdot \prob{\wa^{\top}\xa\leq 0}{\hneg{}\cap \dpos} + p^-  \cdot \prob{\wa^{\top}\xa\geq 0}{\hneg{}\cap \dneg}$ is the internal risk of actionable features.
    
\end{thmlist}
\end{theorem}
Theorem~\ref{Thm::ExpectedCost} implies that one can reduce a worst-case bound on the expected cost of recourse by decreasing the \emph{maximum unit cost of actionable changes} $\ucmax$ or the \emph{internal risk of actionable features} $R_A(f)$. Here, $R_A(f)$ reflects the calibration between the true outcome and the actionable component of the scores $\wa^{\top}\xa$ among individuals where $f(\xb) = -1$. When $R_A(f) = 0$, the actionable component of the scores is perfectly aligned with true outcomes, yielding a tighter bound on the expected cost of recourse.



\section{Integer Programming Tools}
\label{Sec::Methodology}
\newcommand{\recourseIP}{\textsf{RecourseIP}}

In this section, we describe how we solve the optimization problem in \eqref{Eq::RecourseProblem} using integer programming, and discuss how we use this routine to audit recourse and build flipsets.

\subsection{IP Formulation}
\label{Sec::IPFormulation}

We consider a discretized version of the optimization problem in \eqref{Eq::RecourseProblem}, which can be expressed as an \emph{integer program} (IP) and optimized with a solver~\citep[see][for a list]{mittlemanmip2018}. This approach has several benefits: (i) it can directly search over actions for binary, ordinal, and categorical features; (ii) it can optimize non-linear and non-convex cost functions; (iii) it allows users to customize the set of feasible actions; (iv) it can quickly find a globally optimal solution or certify that a classifier does not provide recourse.
\clearpage
We express the optimization problem in \eqref{Eq::RecourseProblem} as an IP of the form:
\begin{subequations}
\label{IP::RecourseIP}
\begin{equationarray}{@{}r@{\;\;}l>{\;}l>{\;}r@{}} 
\min & \textnormal{cost} \notag \\
\st{} & \textnormal{cost} = \sum_{j \in \Ja} \sum_{k=1}^{m_j} c_{jk} v_{jk} & \label{Con::IPCost} \\
& \sum_{j \in \Ja} w_j a_j \geq -\sum_{j=0}^d {w_j x_j}  & \label{Con::IPThreshold} \\
& a_j  =  \sum_{k=1}^{m_j} a_{jk} v_{jk} & j \in \Ja & \label{Con::ActionValue}\\
& 1 = u_j + \sum_{k=1}^{m_j} v_{jk}  &   j \in \Ja  & \label{Con::ActionLimit}\\
& a_j  \in  \R &  j \in \Ja  \notag \\ 
& u_{j} \in \B &  j \in \Ja \notag \\ 
& v_{jk} \in \B & j \in \Ja ,\; \miprange{k}{1}{m_j}\notag
\end{equationarray}
\end{subequations}
Here, constraint \eqref{Con::IPCost} determines the cost of a feasible action using a set of precomputed cost parameters $c_{jk} = \cost{x_{j}+a_{jk}}{x_j}$. Constraint \eqref{Con::IPThreshold} ensures that any feasible action will flip the prediction of a linear classifier with coefficients $\w$. Constraints \eqref{Con::ActionValue} and \eqref{Con::ActionLimit} restrict $a_j$ to a grid of $m_j + 1$ feasible values $a_j \in \{0, a_{j1}, \ldots, a_{jm_j}\}$ via the indicator variables $u_j = 1[a_j = 0]$ and $v_{jk} = 1[a_j = a_{jk}]$. Note that the variables and constraints only depend on actions for actionable features $j \in \Ja$, since $a_j = 0$ when a feature is immutable.

Modern integer programming solvers can quickly recourse a globally optimal solution to IP \eqref{IP::RecourseIP}. In our experiments, for example, CPLEX 12.8 returns a certifiably optimal solution to  \eqref{IP::RecourseIP} or proof of infeasibility within $<0.1$ seconds. In practice, we further reduce solution time through the following changes: (i) we drop the $v_{jk}$ indicators for actions $a_{jk}$ that do not agree in sign with $w_j$; (ii) we declare $\{v_{j1}, \ldots, v_{jm_j}\}$ as a \emph{special ordered set of type I}, which allows the solver to use a more efficient branch-and-bound algorithm \citep{tomlin1988special}. 

\subsubsection*{Customizing the Action Space} 

Users can easily customize the set of feasible actions by adding logical constraints to the IP. These constraints can be used when, for example, a classifier uses dummy variables to encode a categorical attribute (i.e., a one-hot encoding). Many constraints can be expressed with the $u_j$ indicators. For example, we can restrict actions to alter at most one feature within a subset of features $S \subseteq \J$ by adding a constraint of the form $\sum_{j \in S}^d (1 - u_j) \leq 1$. 

\subsubsection*{Discretization Guarantees}

Our IP formulation discretizes the actions for real-valued features so that users can specify a richer class of cost functions. Discretization does not affect the feasibility or the cost of recourse when actions are discretized over a suitably refined grid. We discuss how to build such grids in Appendix~\ref{Appendix::Discretization}.
We can also avoid discretization through an IP formulation that captures the actions of real-valued features using continuous variables. We present this IP formulation in Appendix~\ref{Appendix::ContinuousIP}, but do not discuss it further as it would restrict us to work with linear cost functions.

\subsection{Cost Functions}
\label{Sec::CostFunction}

 
Our IP formulation can optimize a large class of cost functions, including cost functions that may be non-linear or non-convex over the action space. This is because it encodes all values of the cost function in the $c_{jk}$ parameters in constraint \eqref{Con::IPCost}. Formally, our approach requires cost functions that are specified by a vector of values in each actionable dimension. However, it does not necessarily require cost functions that are ``separable" because we can represent some kinds of non-separable functions using minor changes in the IP formulation (see e.g., the cost function in Equation \eqref{Eq::CostFunctionAudit}).

We present off-the-shelf cost functions for our tools in Equations \eqref{Eq::CostFunctionAudit} and \eqref{Eq::CostFunctionFlipsets}. Both functions measure costs in terms of the \emph{percentiles} of $x_j$ and $x_j + a_j$ in the target population: $Q_j(x_j + a_j)$ and $Q_j(x_j)$ where $Q_j(\cdot)$ is the CDF of $x_j$ in the target population. Cost functions based on percentile shifts have the following benefits in comparison to a standard Euclidean distance metric: (i) they do not change with the scale of features; (ii) they reflect the distribution of features in the target population. %
Our functions assign the same cost for a unit percentile change for each feature by default, which assumes that changing each feature is equally difficult. This assumption can be relaxed by, for example, having a domain expert specify the difficulty of changing features relative to a baseline feature. 

\subsection{Auditing Recourse}
\label{Sec::AuditingRecourse}

We evaluate the cost and feasibility of recourse of a linear classifier by solving IP \eqref{IP::RecourseIP} for samples drawn from a population of interest. Formally, the auditing procedure requires: (i) the coefficient vector $\w$ of a linear classifier; (ii) feature vectors sampled from the target population $\{\xb_i\}_{i=1}^n$ where $f(\xb_i) = -1$. It solves the IP for each $\xb_i$ to produce:
\begin{itemize}
\item an estimate of the feasibility of recourse (i.e., the proportion of points for which the IP is feasible); 
\item an estimate of the distribution of the cost of recourse (i.e., the distribution of $\cost{\a^*_i}{\xb_i}$ where ${\a_i}^*$ is the minimal-cost action from $\xb_i$).
\end{itemize}

\subsubsection*{Cost Function} 

We propose the \emph{maximum percentile shift}:
\begin{align}
\cost{\xb+\a}{\xb} = \max_{j \in\Ja} \; \big |Q_j(x_j + a_j) - Q_j(x_j)\big| 
\label{Eq::CostFunctionAudit}.
\end{align}
This cost function is well-suited for auditing because it produces an informative measure of the difficulty of recourse. If the optimal cost is 0.25, for example, then \emph{any} feasible action must change a feature by at least 25 percentiles. In other words, there does not exist an action that flips the prediction by changing a feature by less than 25 percentiles. To run an audit with the cost function in Equation~\eqref{Eq::CostFunctionAudit}, we use a variant of IP~\eqref{IP::RecourseIP} where we replace constraint \eqref{Con::IPCost} with $|\Ja|$ constraints of the form: $\textrm{cost} \geq \sum_{k=1}^{m_j} c_{jk} v_{jk}.$

The maximum percentile shift is also useful for assessing how the feasibility of recourse changes with the bounds of feasible actions. Say that we wanted to assess how many more people have recourse when we assume that each feature can be altered by at most a 50 percentile shift or at most a 90 percentile shift. Using a generic cost function, we would have to compare feasibility estimates from two audits: one where the action set restricts the changes in each feature to a 50 percentile shift, and another where it restricts the changes to a 90 percentile shift. Using the cost function in Equation~\eqref{Eq::CostFunctionAudit}, we only need to run a single audit using a loosely bounded action set (i.e., an action set where each feature can change by 99 percentiles), and compare the number of individuals where the optimal cost exceeds 0.5 and 0.9.

\subsection{Building Flipsets}
\label{Sec::BuildingFlipsets}


We construct flipsets such as the one in Figure \ref{Fig::ExampleFlipset} using \emph{enumeration procedure} that solves IP \eqref{IP::RecourseIP} repeatedly. 

In Algorithm \ref{Alg::BuildFlipset}, we present an enumeration procedure to produce a collection of minimal-cost actions that alter distinct subsets of features. The procedure solves IP \eqref{IP::RecourseIP} to recover a minimal-cost action $\a^*$. Next, it adds a constraint to the IP to eliminate actions that alter the same combination of features as $\a^*$. It repeats these two steps until it has recovered $T$ minimal-cost actions or determined that the IP is infeasible (which means that it has enumerated a minimal-cost action for each combination of features that can flip the prediction from $\xb$).

Each action $\a^* \in \mathcal{A}$ returned by Algorithm \ref{Alg::BuildFlipset} can be used to create an \emph{item} in a flipset by listing the current feature values $x_j$ along with the desired feature values $x_j + {a^*_j}$ for $j \in S = \{j: a^{*}_j \neq 0\}$. 
 
\begin{algorithm}[b]
\begin{algorithmic}[*]
\INPUT
\alginput{\textsf{IP}}{instance of IP \eqref{IP::RecourseIP} for coefficients $\w$, features $\xb$, and actions $\A(\xb)$}
\alginput{$T \geq 1$}{number of items in flipset}
\INITIALIZE
\alginitialize{$\mathcal{A} \gets \emptyset$}{actions shown in flipset}
\Repeat{}
\State $\a^* \gets$ optimal solution to \textsf{IP}
\State $\mathcal{A} \gets \mathcal{A} \cup \{\a^*\}$ \Comment{add $\a^{*}$ to set of optimal actions}
\State $S \gets \{j: {a^{*}_j} \neq 0\}$ \Comment{indices of features altered by $\a^{*}$}
\State add constraint to \textsf{IP} to remove actions that alter features in $S$: $$\sum_{j \not\in S} u_j + \sum_{j \in S} (1 - u_j) \leq d - 1.$$ \label{AlgStep::AddConstraint}
\Until{$|\mathcal{A}| = T$ or \textsf{IP} is infeasible}
\Ensure $\mathcal{A}$ \hfill actions shown in flipset
\end{algorithmic}
\caption{Enumerate $T$ Minimal Cost Actions for Flipset}
\label{Alg::BuildFlipset}
\end{algorithm}

\subsubsection*{Cost Function} We propose the \emph{total log-percentile shift}:
\begin{align}
\cost{\xb+\a}{\xb} = \sum_{j \in \Ja} \log \bigg(\frac{1 - Q_j(x_j + a_j)}{1 - Q_j(x_j)}\bigg) \label{Eq::CostFunctionFlipsets}.
\end{align}
This function aims to produce flipsets where items reflect ``easy" changes in the target population. In particular, it ensures that cost of $a_j$ increases exponentially as $Q_j(x_j) \to 1$. This aims to capture the notion that changes become harder when starting off from a higher percentile value (e.g., changing \textfn{income} from percentiles $90 \to 95$ is harder than $50 \to 55$).

\clearpage
\section{Demonstrations}
\label{Sec::Demonstration}

In this section, we present experiments where we use our tools to study recourse in credit scoring problems. We have two goals: (i) to show how recourse may affected by common practices in the development and deployment of machine learning models; and (ii) to demonstrate how our tools can protect recourse in such events by informing stakeholders such as practitioners, policy-makers and decision-subjects.

In the following experiments, we train classifiers using scikit-learn, and use standard 10-fold cross-validation (10-CV) to tune free parameters and estimate out-of-sample performance. We solve all IPs using the CPLEX 12.8 \cite{cplex} on a 2.6 GHz CPU with 16 GB RAM. We include further information on the features, action sets, and classifiers for each dataset in Appendix \ref{Appendix::Experiments}. We provide scripts to reproduce our analyses at \url{http://github.com/ustunb/actionable-recourse}. 

\subsection{Model Selection}
\label{Sec::Demo1}

\subsubsection*{Setup}

We consider a processed version of \textds{credit} dataset \citep{yeh2009comparisons}. Here, $y_i = -1$ if person $i$ will default on an upcoming credit card payment. The dataset contains $n = 30\,000$ individuals and $d = 16$ features derived from their spending and payment patterns, education, credit history, age, and marital status. We assume that individuals can only change their spending and payment patterns and education.

We train $\ell_1$-penalized logistic regression models for $\ell_1$-penalties of $\{1, 2, 5, 10, 20, 50, 100, 500, 1000\}$. We audit the recourse of each model on the training data, by solving an instance of IP \eqref{IP::RecourseIP} for each $i$ where $\yhat{i} = -1$. The IP includes the following constraints to ensure changes are actionable: (i) changes for discrete features must be discrete (e.g. \textfn{MonthsWithLowSpendingOverLast6Months} $\in \{0,1,\ldots,6\}$); (ii) \textfn{EducationLevel} can only increase; and (iii) immutable features cannot change.

\subsubsection*{Results}

We present the results of our audit in Figure \ref{Fig::CreditRecourseAudit}, and present a flipset for a person who is denied credit by the most accurate classifier in Figure \ref{Fig::CreditFlipset}. 

As shown in Figure \ref{Fig::CreditRecourseAudit}, tuning the $\ell_1$-penalty has a minor effect on test error, but a major effect on the feasibility and cost recourse. In particular, classifiers with small $\ell_1$-penalties provide all individuals with recourse. As the $\ell_1$-penalty increases, however, the number of individuals with recourse decreases as regularization reduces the number of actionable features.

The cost of recourse provides an informative measure of the difficulty of actions. Since we optimize the cost function in Equation \eqref{Eq::CostFunctionAudit}, a cost of $q$ implies a person must change a feature by at least $q$ percentiles to obtain a desired outcome. Here, increasing the $\ell_1$-penalty nearly doubles the median cost of recourse from 0.20 to 0.39. When we deploy a model with a small $\ell_1$-penalty, the median person with recourse can only obtain a desired outcome by changing a feature by at least 20 percentiles. At a large $\ell_1$-penalty, the median person must change a feature by at least 39 percentiles.

Our aim is not to suggest a relationship between recourse and $\ell_1$-regularization, but to show how recourse can be affected by standard tasks in model development such as feature selection and parameter tuning. Here, a practitioner who aims to maximize performance could deploy a model that precludes some individuals from achieving a desired outcome (e.g., the one that minimizes mean 10-CV test error), even as there exist models that perform almost as well but provide all individuals with recourse. 

Our tools can identify mechanisms that affect recourse by running audits with different action sets. For example, one can evaluate how the mutability of feature $j$ affects recourse by running audits for: (i) an action set where feature $j$ is immutable ($\A_j(\xb) = {0}$ for all $\xb \in \X$); and (ii) an action set where feature $j$ is actionable ($\A_j(\xb) = \X_j$ for all $\xb \in \X$). Here, such an analysis reveals that the lack of recourse stems from an immutable feature related to credit history (i.e., an indicator set to 1 if a person has \emph{ever} defaulted on a loan). Given this information, a practitioner could replace this feature with a mutable variant (i.e., an indicator set to 1 if a person has \emph{recently} defaulted on a loan), and thus deploy a model that provides recourse. Such changes are sometimes mandated by industry-specific regulations \cite[see e.g., policies on ``forgetfulness" in][]{blanchette2002data, edwards2017slave}. Our tools can support these efforts by showing how regulations would affect recourse in deployment.

\graphicspath{{/}}
\begin{figure}[h]
\centering
\resizebox{0.5\linewidth}{!}{
\begin{tabular}{@{}lr@{}}
\includegraphics[width=0.215\textwidth]{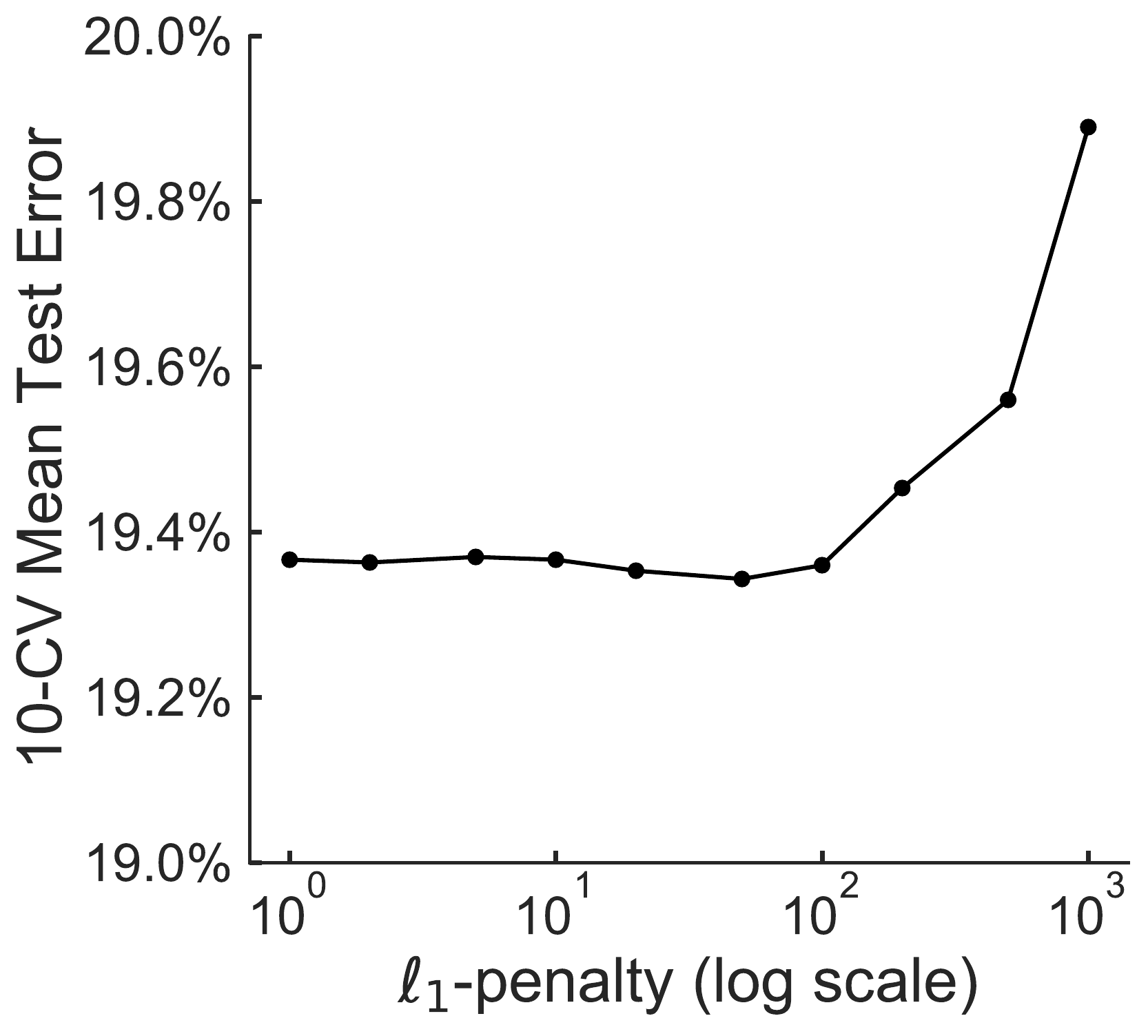} &
\includegraphics[width=0.2\textwidth]{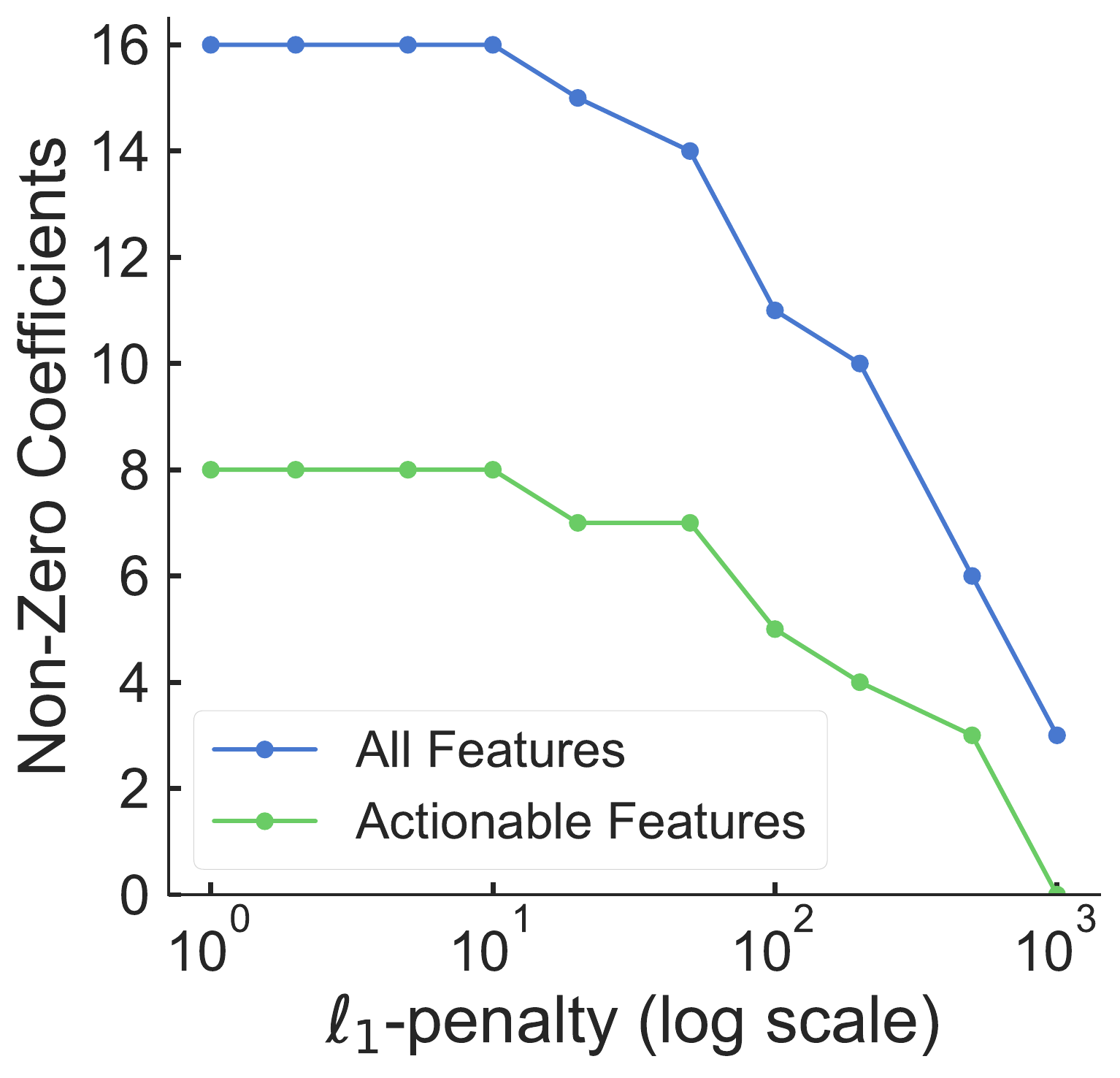}  \\ 
\includegraphics[width=0.215\textwidth]{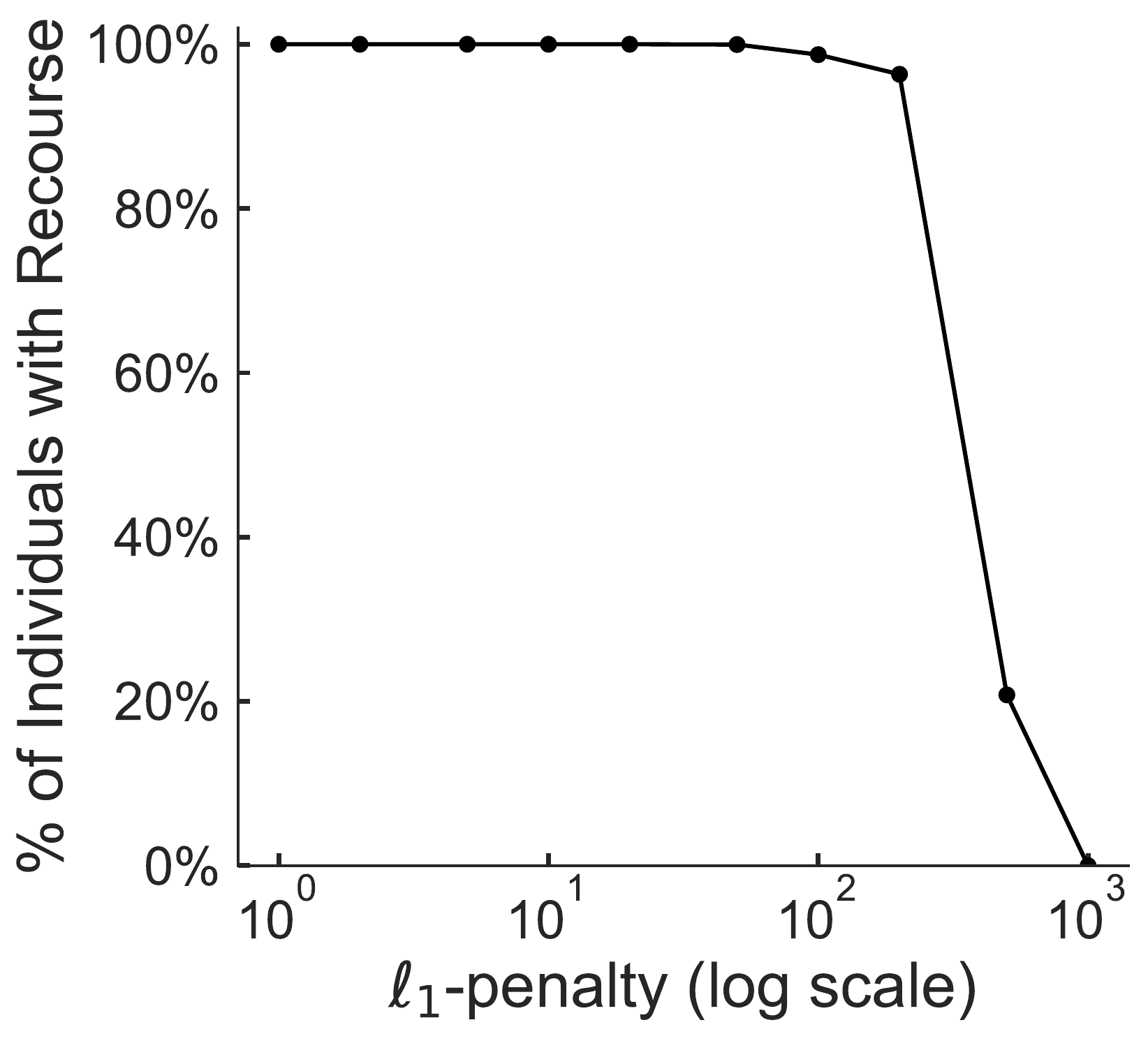} &
\includegraphics[width=0.2\textwidth]{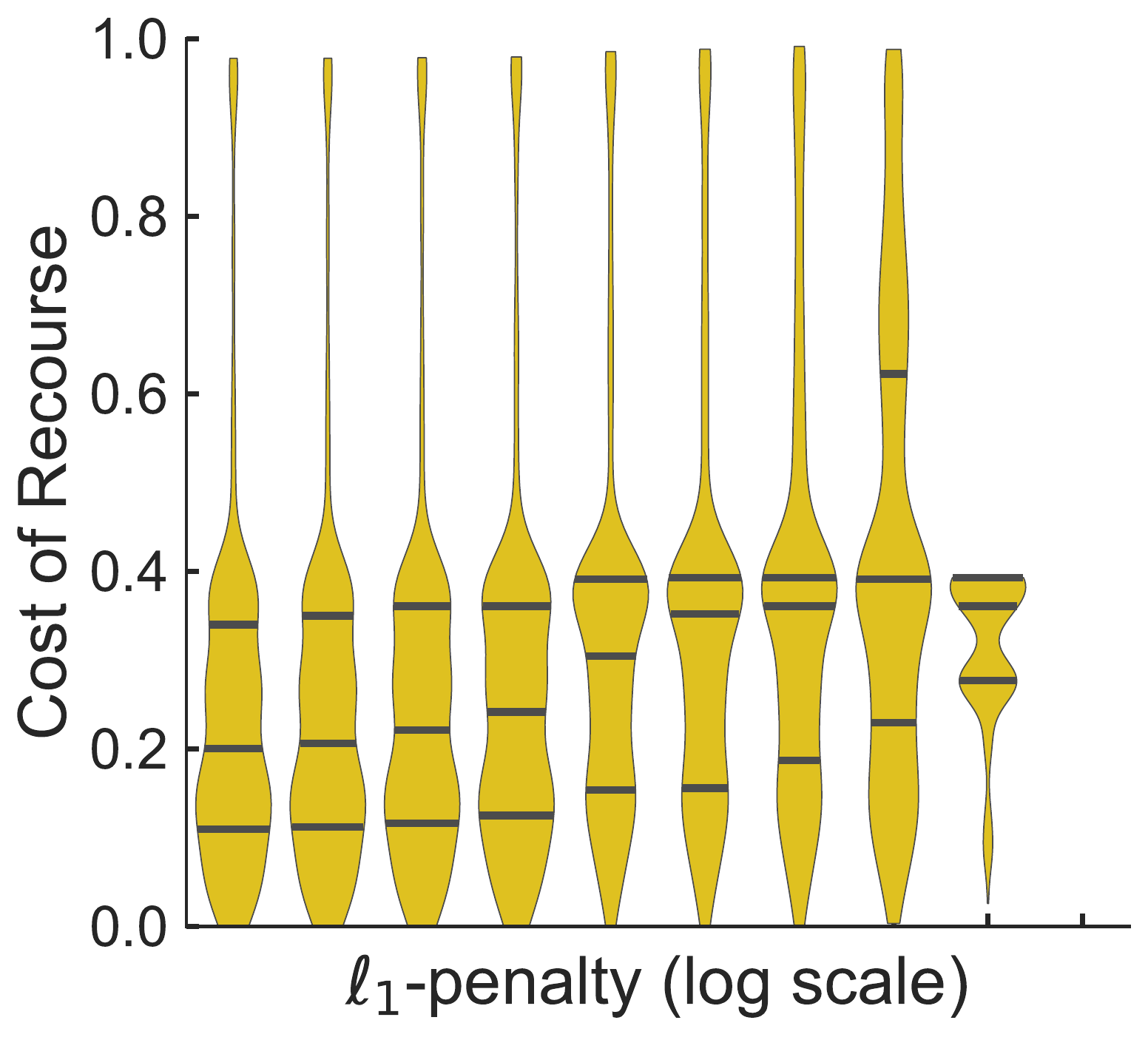}
\end{tabular}
}
\caption{Performance, sparsity, and recourse of $\ell_1$-penalized logistic regression models for the \textds{credit} dataset. We show the mean 10-CV test error (top left), the number of non-zero coefficients (top right),  the proportion of individuals with recourse in the training data (bottom left), and the distribution of the cost of recourse in the training data (bottom right). }
\label{Fig::CreditRecourseAudit}
\end{figure}

\begin{figure}[b]
\centering
\footnotesize
\resizebox{0.6\linewidth}{!}{
\setlength{\tabcolsep}{3pt}
\renewcommand{\arraystretch}{1.05}
\begin{tabular}{rlccc}
\toprule
 &
 \textsc{Features to Change} & 
 \textsc{Current Values} & & 
 \textsc{Required Values} \\ 
 \toprule
& \textit{MostRecentPaymentAmount} &                      \$0 &  $\longrightarrow$ &                     \$790 \\
\midrule
   
&  \textit{MostRecentPaymentAmount} &                      \$0 &  $\longrightarrow$ &                     \$515 \\
&  \textit{MonthsWithZeroBalanceOverLast6Months} &                      1 &  $\longrightarrow$ &                       2 \\
\midrule
&  \textit{MonthsWithZeroBalanceOverLast6Months} &                      1 &  $\longrightarrow$ &                       4 \\
\midrule
 
& \textit{MostRecentPaymentAmount} &                      \$0 &  $\longrightarrow$ &                     \$775 \\
&  \textit{MonthsWithLowSpendingOverLast6Months} &                      6 &  $\longrightarrow$ &                       5 \\

\midrule

& \textit{MostRecentPaymentAmount} &                      \$0 &  $\longrightarrow$ &                     \$500 \\
&  \textit{MonthsWithLowSpendingOverLast6Months} &                      6 &  $\longrightarrow$ &                       5 \\
&  \textit{MonthsWithZeroBalanceOverLast6Months} &                      1 &  $\longrightarrow$ &                       2 \\
\bottomrule
\end{tabular}
}
\caption{Flipset for a person who is denied credit by the most accurate classifier built for the \textds{credit} dataset. Each item shows a minimal-cost action that a person can make to obtain credit.}
\label{Fig::CreditFlipset}
\end{figure}

\FloatBarrier
\subsection{Out-of-Sample Deployment}
\label{Sec::Demo2}

We now discuss an experiment where a classifier is deployed in a setting with dataset shift. Our setup is inspired by a real-world feedback loop with credit scoring in the United States: young adults often lack the credit history to qualify for loans, so they are undersampled in datasets that are used to train a credit score. It is well-known that this kind of systematic undersampling can affect the accuracy of credit scores for young adults~\citep[see e.g.,][]{politico2018creditgap,kallus2018residual}. Here, we show that it can also affect the cost and feasibility of recourse.

\subsubsection*{Setup}

We consider a processed version of the \textds{givemecredit} dataset \citep{data2018givemecredit}. Here, $y_i = -1$ if person $i$ will experience financial distress in the next two years. The data contains $n = 150\,000$ individuals and $d = 10$ features related to their age, dependents, and financial history. We assume that all features are actionable except for \textfn{Age} and \textfn{NumberOfDependents}. 

We draw $n = 112\,500$ examples from the processed dataset to train two $\ell_2$-penalized logistic regression models:
\begin{enumerate}[leftmargin=*]

\item \emph{Baseline Classifier}. This is a baseline model that we train for the sake of comparison. It is trained using all $n = 112\,500$ examples, which represents the target population.

\item \emph{Biased Classifier}. This is the model that we would deploy. It is trained using $n = 98\,120$ examples (i.e.,  the $112\,500$ examples used to train the baseline classifier minus the $14\,380$ examples with $\textfn{Age} < 35$).

\end{enumerate}
We compute the cost of recourse using percentile distributions computed from a hold-out set of $n = 37\,500$ examples. We adjust the threshold of each classifier so that only $10\%$ of examples receive the desired outcome.

\subsubsection*{Results}

We present the results of our audit in Figure~\ref{Fig::OutOfSampleCost} and show flipsets for a prototypical young adult in Figure~\ref{Fig::GiveMeCreditPrototypes}. As shown, the median cost of recourse among young adults under the biased model is 0.66, which means that the median person can only flip their predictions by a 66 percentile shift in any feature. In comparison, the median cost of recourse among young adults under the baseline model is 0.14. These differences in the cost of recourse are less pronounced for other age brackets. 

Our results illustrate how out-of-sample deployment can significantly affect the cost of recourse. In practice, such effects can be measured with an audit using data from a target population. Such a procedure may be useful in model procurement, as classifiers are often deployed on populations that differ from the population that produced the training data.

There are several ways in which out-of-sample deployment can affect recourse. For example, a probabilistic classifier may exhibit a higher cost of recourse if the threshold is fixed, or if the target population has a different set of feasible actions. We controlled for these issues by adjusting thresholds to approve the same proportion of applicants, and by fixing the action set and cost function to audit both classifiers. As a result, the observed effects of out-of-sample deployment only depend on distributional differences in age.

\begin{figure}[t]

{\centering\footnotesize

\begin{tabular}{@{}lcc@{}}

\toprule
& 
\cell{c}{\sc{Individuals where $y=-1$}} & 
\cell{c}{\sc{Individuals where $y=+1$}} \\

\midrule
\cell{l}{\textsc{Baseline}\\\textsc{Classifier}} & 
\cell{c}{\includegraphics[width=0.25\textwidth]{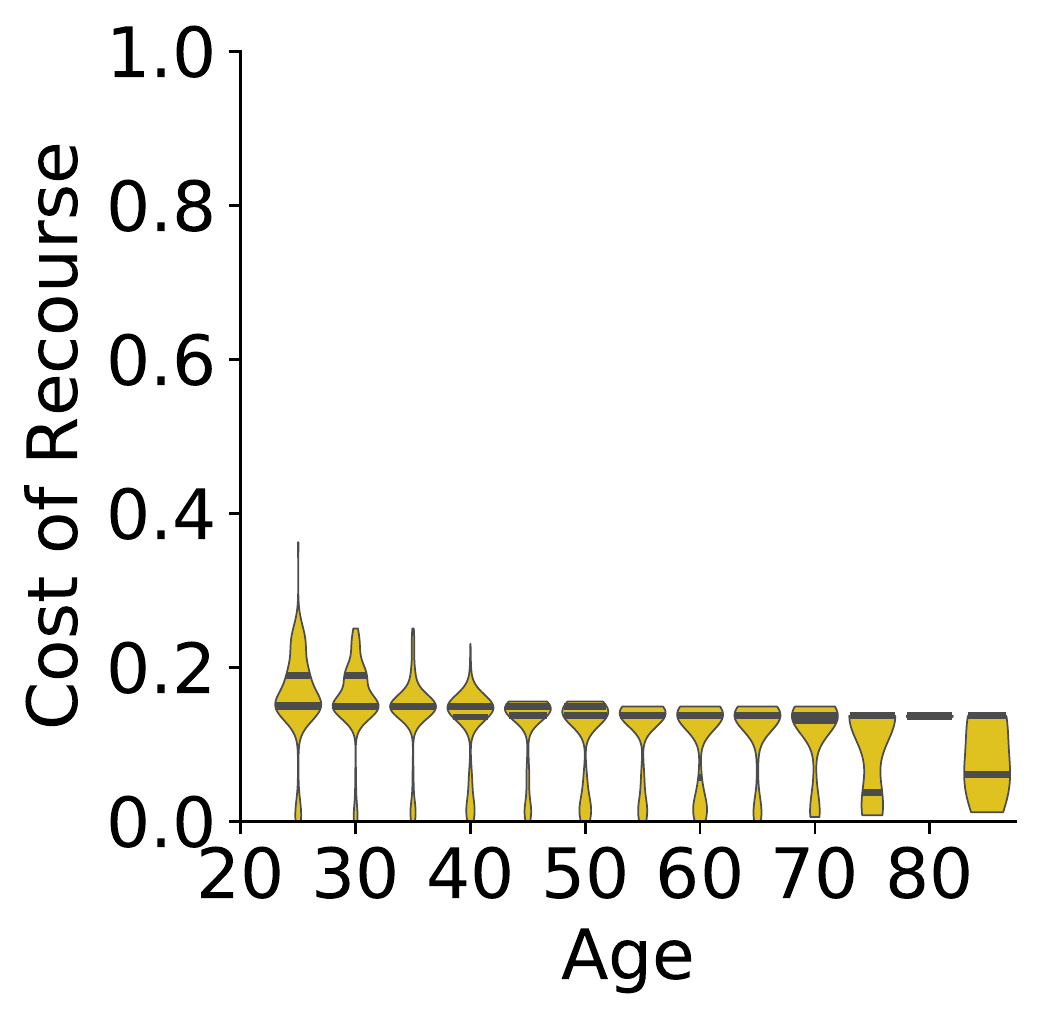}} &
\cell{c}{\includegraphics[width=0.25\textwidth]{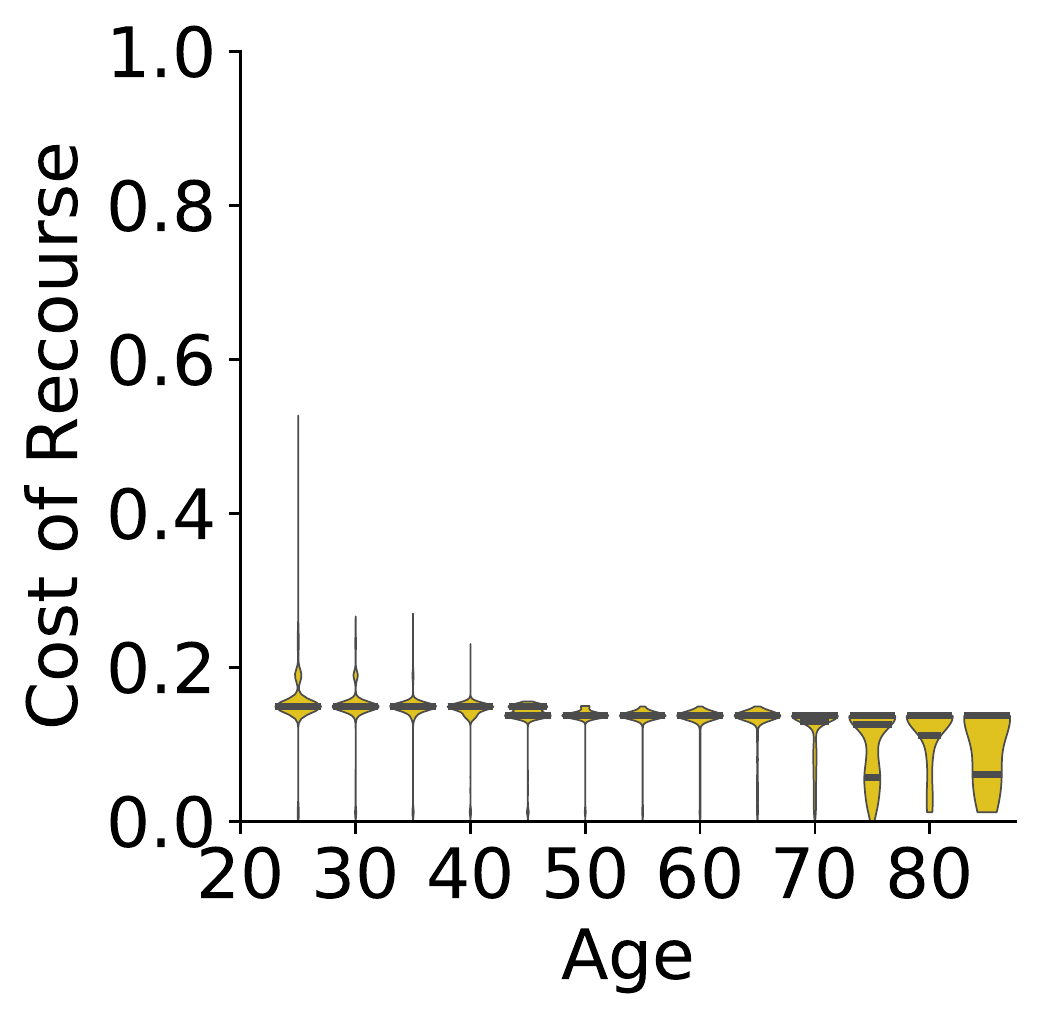}} \\ 
\midrule
\cell{l}{\textsc{Biased}\\\textsc{Classifier}} & 
\cell{c}{\includegraphics[width=0.25\textwidth]{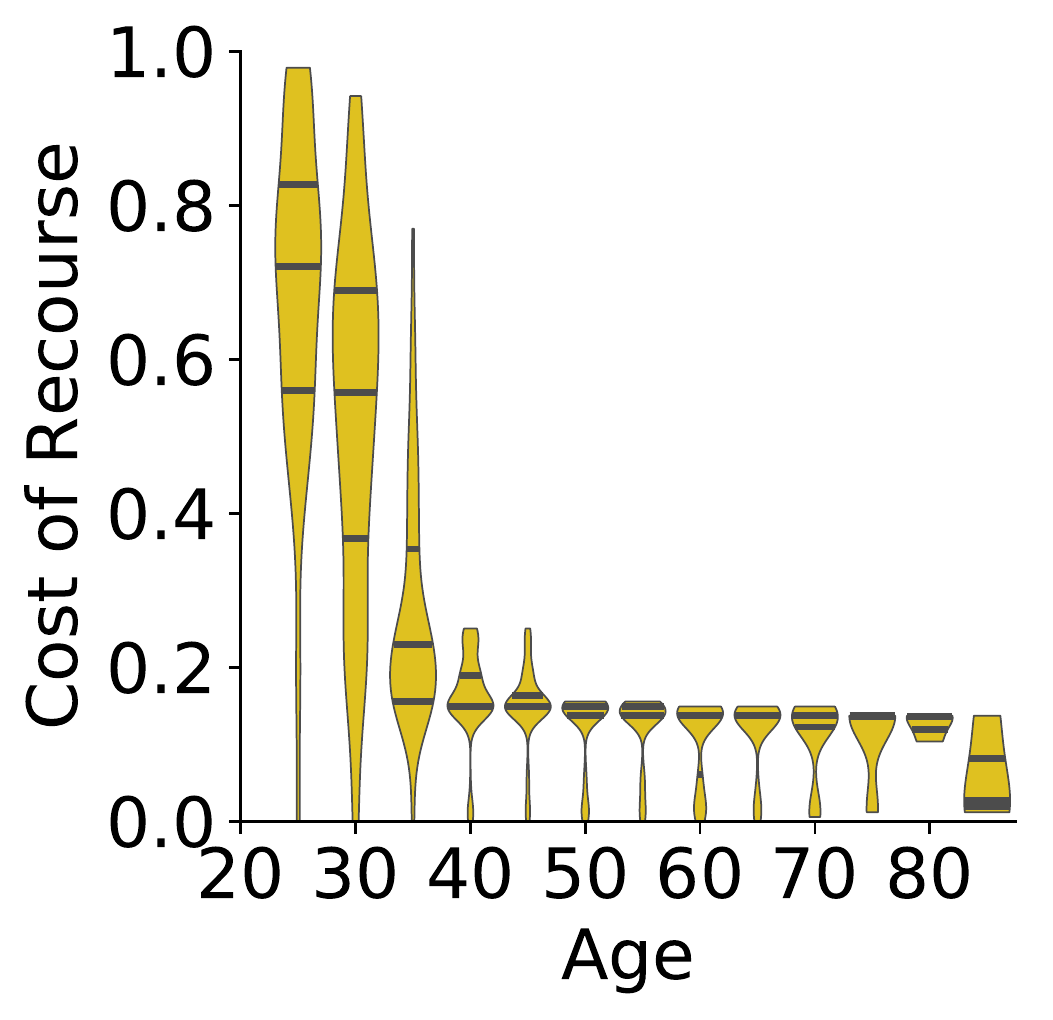}} &
\cell{c}{\includegraphics[width=0.25\textwidth]{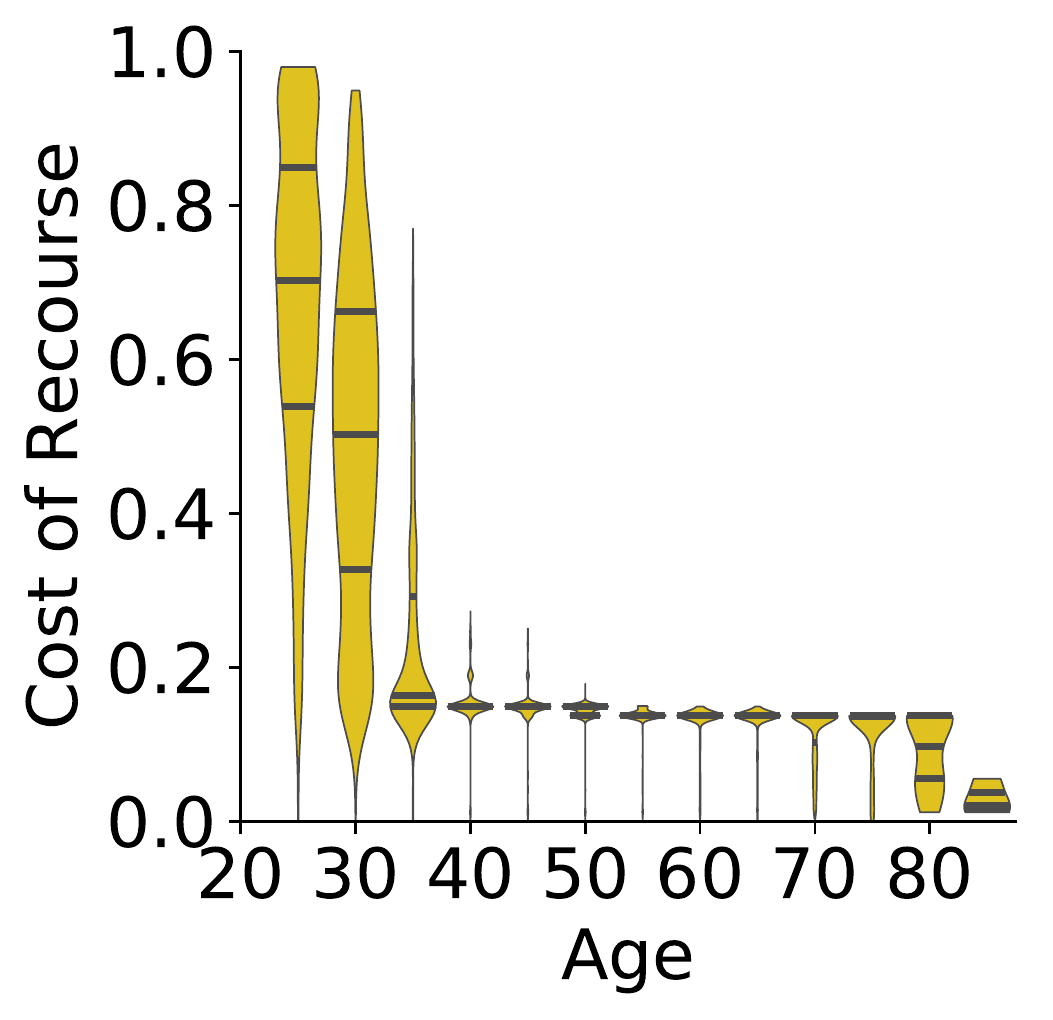}} \\

\bottomrule
\end{tabular}    
}
\caption{Distributions of the cost of recourse in the target population for classifiers conditioned on the true outcome $y$. We show the distribution of the cost of recourse for the biased classifier (top) and the baseline classifier (bottom) for true negatives (left) and false negatives (right). The cost of recourse for young adults is significantly higher for the biased classifier, regardless of their true outcome.}
\label{Fig::OutOfSampleCost}
\end{figure}

\begin{figure}[t]
\centering\footnotesize
\begin{tabular}{@{}c@{}}
\cell{c}{\sc{Baseline Classifier}}\\[0.25em]
\cell{c}{\resizebox{0.6\linewidth}{!}{
\footnotesize\setlength{\tabcolsep}{3pt}\renewcommand{\arraystretch}{1.05}
\begin{tabular}{rlccc}
\toprule &
 \textsc{Features to Change} & 
 \textsc{Current Values} & & 
 \textsc{Required Values} \\ 
\toprule
& \textit{NumberOfTime30-59DaysPastDueNotWorse} & 1 & $\longrightarrow$ & 0 \\
& \textit{NumberOfTime60-89DaysPastDueNotWorse} & 0 & $\longrightarrow$ & 1 \\ 
\midrule
& \textit{NumberRealEstateLoansOrLines} & 2 & $\longrightarrow$ & 1 \\ 
\midrule
& \textit{NumberOfOpenCreditLinesAndLoans} & 11 & $\longrightarrow$ & 12 \\
\midrule
& \textit{RevolvingUtilizationOfUnsecuredLines} & 35.89\% & $\longrightarrow$ & 36.63\% \\
\bottomrule
\end{tabular}}
}\\
\phantom{a}\\
\cell{c}{\sc{Biased Classifier}}\\[0.25em]
\cell{c}{\resizebox{0.6\linewidth}{!}{
\footnotesize\setlength{\tabcolsep}{3pt}\renewcommand{\arraystretch}{1.05}
\begin{tabular}{rlccc}
\toprule &
\textsc{Feature} & 
\textsc{Current Value} & &  
\textsc{Required Value} \\
\toprule
& \textit{NumberOfTime30-59DaysPastDueNotWorse} & 1 & $\longrightarrow$ & 0 \\
& \textit{NumberOfTime60-89DaysPastDueNotWorse} & 0 & $\longrightarrow$ & 1 \\
\bottomrule
\end{tabular}%
}}

\end{tabular}
\caption{Flipsets for a young adult with \textfn{Age} = 28 under the biased classifier (top) and the baseline classifier (bottom). The flipset for the biased classifier has 1 item while the flipset for the baseline classifier has 4 items.}
\label{Fig::GiveMeCreditPrototypes}
\end{figure}

\FloatBarrier
\subsection{Disparities in Recourse}
\label{Sec::Demo3}

Our last experiment aims to illustrate how our tools could be used to evaluate disparities in recourse across protected groups. We evaluate the disparity in recourse of a classifier between males and females while controlling for basic confounding. Here, a disparity in recourse occurs if, given comparable individuals who are denied a loan in the target population, individuals in one group are able to obtain a desired outcome by making easier changes than individuals in another group.

\subsubsection*{Setup}

We consider a processed version of the \textds{german} dataset~\citep{bache2013uci}. Here, $y_i = -1$ if individual $i$ is a ``bad customer." The dataset contains $n = 1\,000$ individuals and $d = 26$ features related to their loan application, financial status, and demographic background. We train a classifier using $\ell_2$-penalized logistic regression, and omit gender from the training data so that the model outputs the identical predictions for male and females with identical features. We evaluate disparities in recourse for this model by examining the cost of recourse for individuals with the same outcome $y$ and similar levels of predicted risk $\prob{y = +1}{}$.

\subsubsection*{Results}

As shown in Figure~\ref{Fig::RecourseViolation}, the cost of recourse can differ between males and females even when models ignore gender. These disparities can also be examined by comparing flipsets as in Figure \ref{Fig::PrototypesGerman}, which shows minimal-cost actions for comparable individuals from each protected group.
\begin{figure}[h]
\begin{tabular}{@{}lr@{}}
\includegraphics[width=0.2\textwidth]{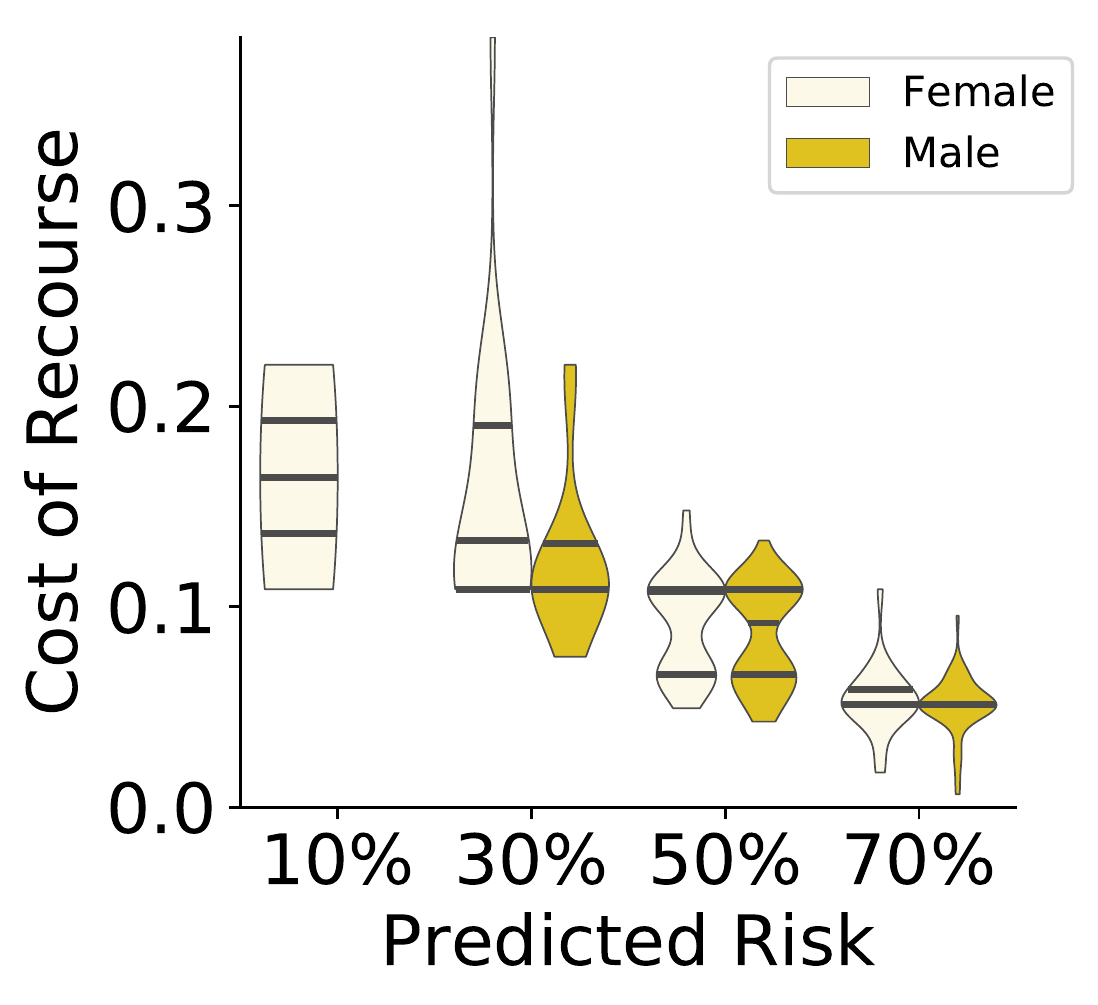} &
\includegraphics[width=0.2\textwidth]{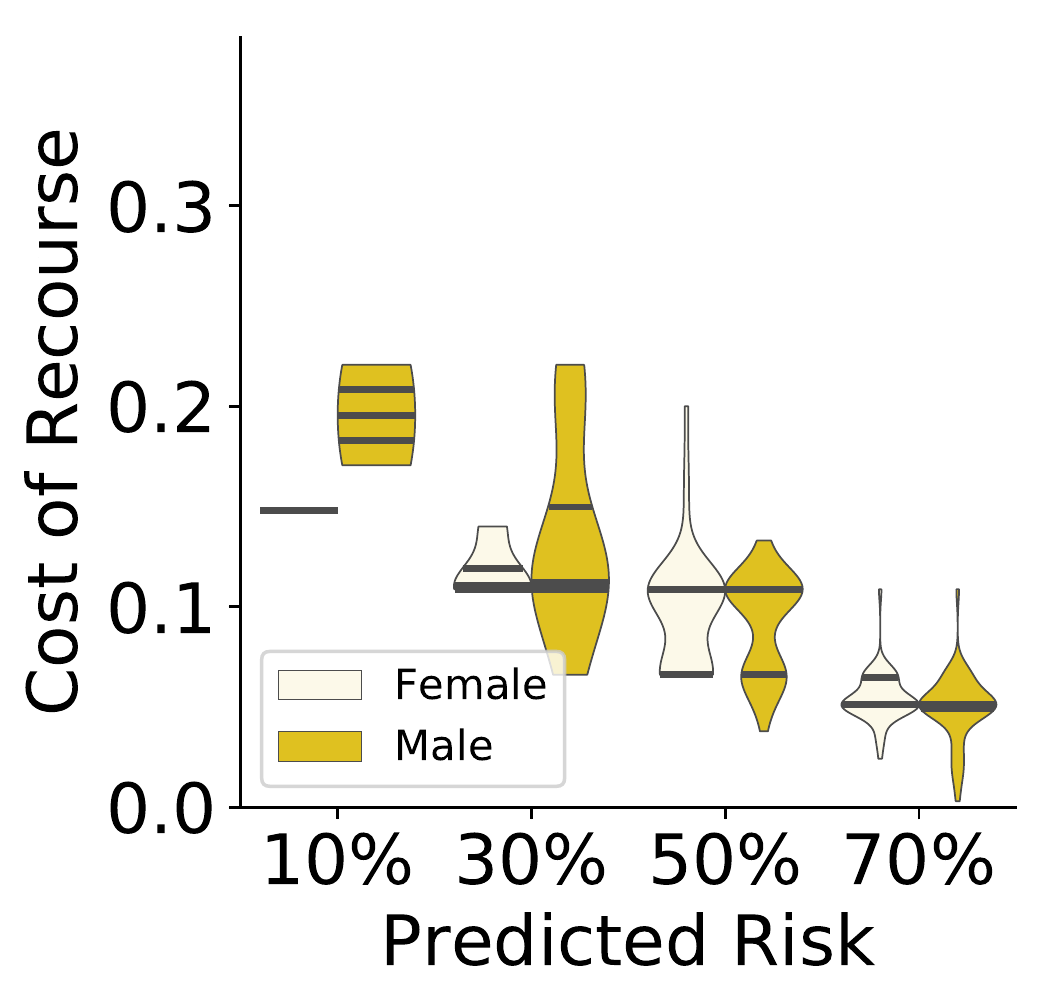}
\end{tabular}
\vspace{-1.5em}
\caption{Distribution of the cost of recourse for males and females with $y = -1$ (left) and $y = +1$ (right).}
\label{Fig::RecourseViolation}
\end{figure}
\vspace{-2em}
\renewcommand{\texttrue}{{\tiny\texttt{TRUE}}}
\renewcommand{\textfalse}{{\tiny\texttt{FALSE}}}

\begin{figure}[h]
\begin{tabular}{@{}cc@{}}
\cell{c}{\scriptsize{\sc{Female}} with $y_i = +1$ and $\textrm{Pr}(y_i = +1) = 34.0\%$} & 
\cell{c}{\scriptsize{\sc{Male}} with $y_i = +1$ and $\textrm{Pr}(y_i =+1) = 32.1\%$} \\
\resizebox{0.4\linewidth}{!}{
\scriptsize\setlength{\tabcolsep}{3pt}\renewcommand{\arraystretch}{1.05}
\begin{tabular}{rlccc}
\toprule
 & 
 \textsc{Features to Change} & 
 \textsc{Current Values} & & 
 \textsc{Required Values} \\ 
 \toprule

& \textit{LoanAmount} & \$7\,432 & $\longrightarrow$ & \$3\,684 \\
& \textit{LoanDuration} & 36 months & $\longrightarrow$ & 25 months \\
& \textit{CheckingAccountBalance $\geq$ 200} & \textfalse{} & $\longrightarrow$ & \texttrue{} \\
& \textit{SavingsAccountBalance $\geq$ 100} & \textfalse{} & $\longrightarrow$ & \texttrue{} \\
& \textit{HasGuarantor} & \textfalse{} & $\longrightarrow$ & \texttrue{} \\

\midrule

& \textit{LoanAmount} & \$7\,432 & $\longrightarrow$ & \$3,684 \\
& \textit{LoanDuration} & 36 months & $\longrightarrow$ & 23 months \\
& \textit{LoanRateAsPercentOfIncome} & 2.00\% & $\longrightarrow$ & 1.00\% \\
& \textit{HasTelephone} & \textfalse{} & $\longrightarrow$ & \texttrue{} \\
& \textit{HasGuarantor} & \textfalse{} & $\longrightarrow$ & \texttrue{} \\

\midrule

& \textit{LoanAmount} & \$7432 & $\longrightarrow$ & \$912 \\
& \textit{LoanDuration} & 36 months & $\longrightarrow$ & 7 months \\
& \textit{HasTelephone} & \textfalse{} & $\longrightarrow$ & \texttrue{} \\ 
\bottomrule
\vspace{6em}
\end{tabular}}
%
& 
\cell{c}{%
\resizebox{0.4\linewidth}{!}{
\footnotesize
\setlength{\tabcolsep}{3pt}
\renewcommand{\arraystretch}{1.05}
\begin{tabular}{rlccc}
\toprule
 & 
 \textsc{Features to Change} & 
 \textsc{Current Values} & & 
 \textsc{Required Values} \\ 
 \toprule
 
& \textit{LoanAmount} & \$15\,857 & $\longrightarrow$ & \$7\,968 \\
& \textit{LoanDuration} & 36 months & $\longrightarrow$ & 32 months \\

& \textit{CheckingAccountBalance $\geq$ 200} & \textfalse{} & $\longrightarrow$ & \texttrue{}  \\

& \textit{HasCoapplicant} & \texttrue{} & $\longrightarrow$ & \textfalse{} \\

& \textit{HasGuarantor} & \textfalse{} & $\longrightarrow$ & \texttrue{} \\

& \textit{Unemployed} & \texttrue{} & $\longrightarrow$ & \textfalse{} \\

\midrule

& \textit{LoanAmount} & \$15\,857 & $\longrightarrow$ & \$7\,086 \\

& \textit{LoanDuration} & 36 months & $\longrightarrow$ & 29 months \\

& \textit{CheckingAccountBalance $\geq$ 200} & \textfalse{} & $\longrightarrow$  & \texttrue{} \\

& \textit{HasCoapplicant} & \texttrue{} & $\longrightarrow$  & \textfalse{} \\

& \textit{HasGuarantor} & \textfalse{} & $\longrightarrow$  & \texttrue{} \\

\midrule

& \textit{LoanAmount} & \$15\,857 & $\longrightarrow$  & \$4\,692 \\

& \textit{LoanDuration} & 36 months & $\longrightarrow$  & 29 months \\

& \textit{CheckingAccountBalance $\geq$ 200} & \textfalse{} & $\longrightarrow$  & \texttrue{} \\

& \textit{SavingsAccountBalance $\geq$ 100} & \textfalse{} & $\longrightarrow$  & \texttrue{} \\

\midrule

& \textit{LoanAmount} & \$15\,857 & $\longrightarrow$ & \$3\,684 \\

& \textit{LoanDuration} & 36 months & $\longrightarrow$ & 21 months \\

& \textit{HasTelephone} & \textfalse{} & $\longrightarrow$ & \texttrue{}\\

\bottomrule
\end{tabular}}
}

\end{tabular}

\caption{Flipsets for a matched pair of individuals from each protected group. Individuals have the same true outcome $y_i$ and similar levels predicted risk $\prob{y_i = +1}{}$.}
\label{Fig::PrototypesGerman}
\end{figure}

\FloatBarrier
\section{Concluding Remarks}
\label{Sec::Discussion}


\subsection{Extensions}
\label{Sec::Extensions}

\subsubsection*{Non-Linear Classifiers} 

We are currently extending our tools to evaluate recourse for non-linear classifiers. One could apply our tools to this setting by replacing the linear classifier with a local linear model that approximates the decision boundary around $\xb$ in actionable space \citep[e.g., similar to the approach used by LIME,][]{ribeiro2016should}. This approach may find actionable changes. However, it would not provide the proof of infeasibility that is required to claim that a model does not provide recourse.

\subsubsection*{Pricing Incentives}

Our tools could be used to price incentives induced by a model by running audits with different action sets \citep[see e.g.,][]{kleinberg2018strategic}. Consider a credit score that includes features that are causally related to creditworthiness (e.g., income) and ``ancillary" features that are prone to manipulation (e.g., social media presence). In this case, one could price incentives in a target population by comparing the cost of recourse for actions that alter (i) only causal features, and (ii) causal features and at least one ancillary feature. 

\subsubsection*{Measuring Flexibility} 

Our tools can enumerate the complete set of minimal-cost actions for a person by using the procedure in Algorithm \ref{Alg::BuildFlipset} to list actions until the IP become infeasible. This would produce a collection of actions, where each action reflects a way to obtain the outcome by altering a different subset of features. The size of this collection would reflect the flexibility of recourse, and may be used to evaluate other aspects of recourse. For example, if a classifier provides a person with 16 ways to flip their prediction, 15 of which are legally contestable, then the model itself may be contestable.

\subsection{Limitations}
\label{Sec::Limitations}

\subsubsection*{Abridged Flipsets}

The flipsets in this paper are ``abridged" in that they do not reveal all features of the model. In practice, a person who is shown a flipset in this format may fail to flip their prediction after making the recommended changes if they unknowingly alter actionable features that are not shown. This issue can be avoided by including additional information along the flipset -- for example, a list of undisclosed actionable features that must not change, or a list of features that must change in a certain way. Alternatively, one could also build an abridged flipset with ``robust" actions (i.e., actions that flip the prediction \emph{and} provide an additional ``buffer" to protect against the possibility that a person alters other undisclosed features in an adversarial manner).

\subsubsection*{Model Theft}

Model owners may not be willing to provide consumers with flipsets due to the potential for model theft (see e.g., efforts to reverse-engineer the Schufa credit score in Germany by crowdsourcing \cite{schufa2018}). One way to address such concerns would be to produce a lower bound on the number of actions needed to reconstruct a proprietary model \cite{tramer2016stealing,milli2019model}. Such bounds may be useful for quantifying the risk of model theft, and to inform the design of safeguards to reduce this risk.

\subsection{Discussion}
\label{Sec::BroaderDiscussion}

\subsubsection*{When Should Models Provide Recourse?}

Individual rights with regards to algorithmic decision-making are often motivated by the need for human agency over machine-made decisions. Recourse reflects a precise notion of human agency -- i.e., the ability of a person to alter the predictions of a model. We argue that models should provide recourse in applications subject to equal opportunity laws (e.g., lending or hiring) and in applications where individuals should have agency over decisions (e.g., the allocation of social services). 

Recourse provides a useful concept to articulate notions of \emph{procedural fairness} in applications without a universal ``right to recourse." In recidivism prediction, for example, models should provide defendants who are predicted to recidivate with the ability to flip their prediction by altering specific sets of features. For example, a model that includes age and criminal history should allow defendants who are predicted to recidivate to flip their predictions by ``clearing their criminal history." Otherwise, some defendants would be predicted to recidivate solely on the basis of age.


In settings where recourse is desirable, our tools can check that a model provides recourse through two approaches: (1) by running periodic recourse audits; and (2) by generating a flipset for every person who is assigned an undesirable prediction. The second approach has a benefit in that it can detect recourse violations while a model is deployed. That is, we would know that a model did not provide recourse to all its decision subjects on the first instance that we would produce an empty flipset.

\subsubsection*{Should We Inform Consumers of Recourse?}

In settings where there is an imperative for recourse, providing consumers with flipsets may lead to harm.
Consider a case where a consumer is denied a loan by a model so that $\yhat{} = -1$, and we know that they are likely to default so that $y = -1$.
In this case, providing them with an action that allows them to flip their prediction from $\yhat{} = -1$ to $\yhat{} = +1$ could inflict harm if the action did not also improve their ability to repay the loan from $y = -1$ to $y = +1.$
Conversely, say that we presented the consumer with an action that allowed them to receive a loan and also improved their ability to pay it back. In this case, disclosing the action would benefit both parties: the consumer would receive a loan that they could repay, and the model owner would have improved the creditworthiness of their consumer.

This example shows how flipsets could benefit all parties if we can find actions that simultaneously alter their predicted outcome $\yhat{}$ and true outcome $y$. Such actions are naturally produced by causal models. They could also be obtained for predictive models. For example, we could enumerate all actions that flip the predicted outcome $\yhat{}$, then build a filtered flipset using actions that are likely to flip a true outcome $y$ (e.g., using a technique to estimate treatment effects from observational data).

In practice, the potential drawbacks of gaming have not stopped the development of laws and tools to empower consumers with actionable information. In the United States, for example, the adverse action requirement of the Equal Credit Opportunity Act is designed -- in part -- to educate consumers on how to obtain credit~\citep[see e.g.,][for a discussion]{taylor1980meeting}. In addition, credit bureaus provide credit score simulators that allow consumers to find actions that will change their credit score in a desired way.%
\footnote{See, for example, \url{https://www.transunion.com/product/credit-score-simulator.}}

\subsubsection*{Policy Implications}

While regulations for algorithmic decision-making are still in their infancy, existing efforts have sought to ensure human agency indirectly, through laws that focus on transparency and explanation \citep[see e.g., regulations for credit scoring in the United States such as][]{congress2003facta}. In light of these efforts, we argue that recourse should be treated as a standalone policy goal when it is desirable. This is because recourse is a precise concept with several options for meaningful consumer protection. For example, one could mandate that a classifier must be paired with a recourse audit for its target population, or mandate that consumers who are assigned an undesirable outcome are shown a list of actions to obtain a desired outcome.


%


\begin{acks}
We thank the following individuals for helpful discussions: Solon Barocas, Flavio Calmon, Yaron Singer, Ben Green, Hao Wang, Suresh Venkatasubramanian, Sharad Goel, Matt Weinberg, Aloni Cohen, Jesse Engreitz, and Margaret Haffey.
\end{acks}

\bibliographystyle{ACM-Reference-Format}
\normalsize
\bibliography{actionable_recourse}


\begin{thebibliography}{52}


\ifx \showCODEN    \undefined \def \showCODEN     #1{\unskip}     \fi
\ifx \showDOI      \undefined \def \showDOI       #1{#1}\fi
\ifx \showISBNx    \undefined \def \showISBNx     #1{\unskip}     \fi
\ifx \showISBNxiii \undefined \def \showISBNxiii  #1{\unskip}     \fi
\ifx \showISSN     \undefined \def \showISSN      #1{\unskip}     \fi
\ifx \showLCCN     \undefined \def \showLCCN      #1{\unskip}     \fi
\ifx \shownote     \undefined \def \shownote      #1{#1}          \fi
\ifx \showarticletitle \undefined \def \showarticletitle #1{#1}   \fi
\ifx \showURL      \undefined \def \showURL       {\relax}        \fi
\providecommand\bibfield[2]{#2}
\providecommand\bibinfo[2]{#2}
\providecommand\natexlab[1]{#1}
\providecommand\showeprint[2][]{arXiv:#2}

\bibitem[\protect\citeauthoryear{Aggarwal, Chen, and Han}{Aggarwal
  et~al\mbox{.}}{2010}]%
        {aggarwal2010inverse}
\bibfield{author}{\bibinfo{person}{Charu~C Aggarwal}, \bibinfo{person}{Chen
  Chen}, {and} \bibinfo{person}{Jiawei Han}.} \bibinfo{year}{2010}\natexlab{}.
\newblock \showarticletitle{The Inverse Classification Problem}.
\newblock \bibinfo{journal}{\emph{Journal of Computer Science and Technology}}
  \bibinfo{volume}{25}, \bibinfo{number}{3} (\bibinfo{year}{2010}),
  \bibinfo{pages}{458--468}.
\newblock


\bibitem[\protect\citeauthoryear{Ajunwa, Friedler, Scheidegger, and
  Venkatasubramanian}{Ajunwa et~al\mbox{.}}{2016}]%
        {ajunwa2016hiring}
\bibfield{author}{\bibinfo{person}{Ifeoma Ajunwa}, \bibinfo{person}{Sorelle
  Friedler}, \bibinfo{person}{Carlos~E Scheidegger}, {and}
  \bibinfo{person}{Suresh Venkatasubramanian}.}
  \bibinfo{year}{2016}\natexlab{}.
\newblock \showarticletitle{Hiring by Algorithm: Predicting and Preventing
  Disparate Impact}.
\newblock \bibinfo{journal}{\emph{Available at SSRN}} (\bibinfo{year}{2016}).
\newblock


\bibitem[\protect\citeauthoryear{Angelino, Larus-Stone, Alabi, Seltzer, and
  Rudin}{Angelino et~al\mbox{.}}{2017}]%
        {angelino2017learning}
\bibfield{author}{\bibinfo{person}{Elaine Angelino}, \bibinfo{person}{Nicholas
  Larus-Stone}, \bibinfo{person}{Daniel Alabi}, \bibinfo{person}{Margo
  Seltzer}, {and} \bibinfo{person}{Cynthia Rudin}.}
  \bibinfo{year}{2017}\natexlab{}.
\newblock \showarticletitle{Learning certifiably optimal rule lists}. In
  \bibinfo{booktitle}{\emph{Proceedings of the 23rd ACM SIGKDD International
  Conference on Knowledge Discovery and Data Mining}}. ACM,
  \bibinfo{pages}{35--44}.
\newblock


\bibitem[\protect\citeauthoryear{Bache and Lichman}{Bache and Lichman}{2013}]%
        {bache2013uci}
\bibfield{author}{\bibinfo{person}{Kevin Bache} {and} \bibinfo{person}{Moshe
  Lichman}.} \bibinfo{year}{2013}\natexlab{}.
\newblock \bibinfo{title}{{UCI Machine Learning Repository}}.
\newblock
\newblock


\bibitem[\protect\citeauthoryear{Belotti, Bonami, Fischetti, Lodi, Monaci,
  Nogales-G{\'o}mez, and Salvagnin}{Belotti et~al\mbox{.}}{2016}]%
        {belotti2016handling}
\bibfield{author}{\bibinfo{person}{Pietro Belotti}, \bibinfo{person}{Pierre
  Bonami}, \bibinfo{person}{Matteo Fischetti}, \bibinfo{person}{Andrea Lodi},
  \bibinfo{person}{Michele Monaci}, \bibinfo{person}{Amaya Nogales-G{\'o}mez},
  {and} \bibinfo{person}{Domenico Salvagnin}.} \bibinfo{year}{2016}\natexlab{}.
\newblock \showarticletitle{On handling indicator constraints in mixed integer
  programming}.
\newblock \bibinfo{journal}{\emph{Computational Optimization and Applications}}
  \bibinfo{volume}{65}, \bibinfo{number}{3} (\bibinfo{year}{2016}),
  \bibinfo{pages}{545--566}.
\newblock


\bibitem[\protect\citeauthoryear{Binns, Van~Kleek, Veale, Lyngs, Zhao, and
  Shadbolt}{Binns et~al\mbox{.}}{2018}]%
        {binns2018s}
\bibfield{author}{\bibinfo{person}{Reuben Binns}, \bibinfo{person}{Max
  Van~Kleek}, \bibinfo{person}{Michael Veale}, \bibinfo{person}{Ulrik Lyngs},
  \bibinfo{person}{Jun Zhao}, {and} \bibinfo{person}{Nigel Shadbolt}.}
  \bibinfo{year}{2018}\natexlab{}.
\newblock \showarticletitle{'It's Reducing a Human Being to a Percentage':
  Perceptions of Justice in Algorithmic Decisions}. In
  \bibinfo{booktitle}{\emph{Proceedings of the 2018 CHI Conference on Human
  Factors in Computing Systems}}. ACM, \bibinfo{pages}{377}.
\newblock


\bibitem[\protect\citeauthoryear{Biran and McKeown}{Biran and McKeown}{2014}]%
        {biran2014justification}
\bibfield{author}{\bibinfo{person}{Or Biran} {and} \bibinfo{person}{Kathleen
  McKeown}.} \bibinfo{year}{2014}\natexlab{}.
\newblock \showarticletitle{Justification narratives for individual
  classifications}. In \bibinfo{booktitle}{\emph{Proceedings of the AutoML
  workshop at ICML}}, Vol.~\bibinfo{volume}{2014}.
\newblock


\bibitem[\protect\citeauthoryear{Blanchette and Johnson}{Blanchette and
  Johnson}{2002}]%
        {blanchette2002data}
\bibfield{author}{\bibinfo{person}{Jean-Fran{\c{c}}ois Blanchette} {and}
  \bibinfo{person}{Deborah~G Johnson}.} \bibinfo{year}{2002}\natexlab{}.
\newblock \showarticletitle{Data retention and the panoptic society: The social
  benefits of forgetfulness}.
\newblock \bibinfo{journal}{\emph{The Information Society}}
  \bibinfo{volume}{18}, \bibinfo{number}{1} (\bibinfo{year}{2002}),
  \bibinfo{pages}{33--45}.
\newblock


\bibitem[\protect\citeauthoryear{Bogen and Rieke}{Bogen and Rieke}{2018}]%
        {bogen2018hiring}
\bibfield{author}{\bibinfo{person}{Miranda Bogen} {and} \bibinfo{person}{Aaron
  Rieke}.} \bibinfo{year}{2018}\natexlab{}.
\newblock \showarticletitle{Help wanted: an examination of hiring algorithms,
  equity, and bias}.
\newblock  (\bibinfo{year}{2018}).
\newblock
\urldef\tempurl%
\url{https://www.upturn.org/reports/2018/hiring-algorithms/}
\showURL{%
\tempurl}


\bibitem[\protect\citeauthoryear{Chang, Rudin, Cavaretta, Thomas, and
  Chou}{Chang et~al\mbox{.}}{2012}]%
        {chang2012reverse}
\bibfield{author}{\bibinfo{person}{Allison Chang}, \bibinfo{person}{Cynthia
  Rudin}, \bibinfo{person}{Michael Cavaretta}, \bibinfo{person}{Robert Thomas},
  {and} \bibinfo{person}{Gloria Chou}.} \bibinfo{year}{2012}\natexlab{}.
\newblock \showarticletitle{How to reverse-engineer quality rankings}.
\newblock \bibinfo{journal}{\emph{Machine Learning}} \bibinfo{volume}{88},
  \bibinfo{number}{3} (\bibinfo{year}{2012}), \bibinfo{pages}{369--398}.
\newblock


\bibitem[\protect\citeauthoryear{Chouldechova, Benavides-Prado, Fialko, and
  Vaithianathan}{Chouldechova et~al\mbox{.}}{2018}]%
        {chouldechova2018case}
\bibfield{author}{\bibinfo{person}{Alexandra Chouldechova},
  \bibinfo{person}{Diana Benavides-Prado}, \bibinfo{person}{Oleksandr Fialko},
  {and} \bibinfo{person}{Rhema Vaithianathan}.}
  \bibinfo{year}{2018}\natexlab{}.
\newblock \showarticletitle{A Case Study of Algorithm-Assisted Decision Making
  in Child Maltreatment Hotline Screening Decisions}. In
  \bibinfo{booktitle}{\emph{Conference on Fairness, Accountability and
  Transparency}}. \bibinfo{pages}{134--148}.
\newblock


\bibitem[\protect\citeauthoryear{Citron and Pasquale}{Citron and
  Pasquale}{2014}]%
        {citron2014scored}
\bibfield{author}{\bibinfo{person}{Danielle~Keats Citron} {and}
  \bibinfo{person}{Frank Pasquale}.} \bibinfo{year}{2014}\natexlab{}.
\newblock \showarticletitle{The Scored Society: Due Process for Automated
  Predictions}.
\newblock \bibinfo{journal}{\emph{Washington Law Review}}  \bibinfo{volume}{89}
  (\bibinfo{year}{2014}), \bibinfo{pages}{1}.
\newblock


\bibitem[\protect\citeauthoryear{Cowgill and Tucker}{Cowgill and
  Tucker}{2019}]%
        {cowgill2019economics}
\bibfield{author}{\bibinfo{person}{Bo Cowgill} {and}
  \bibinfo{person}{Catherine~E Tucker}.} \bibinfo{year}{2019}\natexlab{}.
\newblock \showarticletitle{Economics, fairness and algorithmic bias}.
\newblock  (\bibinfo{year}{2019}).
\newblock


\bibitem[\protect\citeauthoryear{Crawford and Schultz}{Crawford and
  Schultz}{2014}]%
        {crawford2014big}
\bibfield{author}{\bibinfo{person}{Kate Crawford} {and} \bibinfo{person}{Jason
  Schultz}.} \bibinfo{year}{2014}\natexlab{}.
\newblock \showarticletitle{Big data and due process: Toward a framework to
  redress predictive privacy harms}.
\newblock \bibinfo{journal}{\emph{BCL Rev.}}  \bibinfo{volume}{55}
  (\bibinfo{year}{2014}), \bibinfo{pages}{93}.
\newblock


\bibitem[\protect\citeauthoryear{Deutschland}{Deutschland}{2018}]%
        {schufa2018}
\bibfield{author}{\bibinfo{person}{Open Knowledge~Foundation Deutschland}.}
  \bibinfo{year}{2018}\natexlab{}.
\newblock \bibinfo{title}{Get Involved: We Crack the Schufa!}
\newblock
  \bibinfo{howpublished}{\url{https://okfn.de/blog/2018/02/openschufa-english/}}.
\newblock


\bibitem[\protect\citeauthoryear{Dong, Roth, Schutzman, Waggoner, and Wu}{Dong
  et~al\mbox{.}}{2018}]%
        {dong2018strategic}
\bibfield{author}{\bibinfo{person}{Jinshuo Dong}, \bibinfo{person}{Aaron Roth},
  \bibinfo{person}{Zachary Schutzman}, \bibinfo{person}{Bo Waggoner}, {and}
  \bibinfo{person}{Zhiwei~Steven Wu}.} \bibinfo{year}{2018}\natexlab{}.
\newblock \showarticletitle{Strategic Classification from Revealed
  Preferences}. In \bibinfo{booktitle}{\emph{Proceedings of the 2018 ACM
  Conference on Economics and Computation}}. ACM, \bibinfo{pages}{55--70}.
\newblock


\bibitem[\protect\citeauthoryear{{Doshi-Velez}, {Kortz}, {Budish}, {Bavitz},
  {Gershman}, {O'Brien}, {Schieber}, {Waldo}, {Weinberger}, and
  {Wood}}{{Doshi-Velez} et~al\mbox{.}}{2017}]%
        {doshi2017accountability}
\bibfield{author}{\bibinfo{person}{Finale {Doshi-Velez}},
  \bibinfo{person}{Mason {Kortz}}, \bibinfo{person}{Ryan {Budish}},
  \bibinfo{person}{Chris {Bavitz}}, \bibinfo{person}{Sam {Gershman}},
  \bibinfo{person}{David {O'Brien}}, \bibinfo{person}{Stuart {Schieber}},
  \bibinfo{person}{James {Waldo}}, \bibinfo{person}{David {Weinberger}}, {and}
  \bibinfo{person}{Alexandra {Wood}}.} \bibinfo{year}{2017}\natexlab{}.
\newblock \showarticletitle{{Accountability of AI Under the Law: The Role of
  Explanation}}.
\newblock \bibinfo{journal}{\emph{ArXiv e-prints}}, Article
  \bibinfo{articleno}{arXiv:1711.01134} (\bibinfo{date}{Nov.}
  \bibinfo{year}{2017}).
\newblock
\showeprint[arxiv]{1711.01134}


\bibitem[\protect\citeauthoryear{Edwards and Veale}{Edwards and Veale}{2017}]%
        {edwards2017slave}
\bibfield{author}{\bibinfo{person}{Lilian Edwards} {and}
  \bibinfo{person}{Michael Veale}.} \bibinfo{year}{2017}\natexlab{}.
\newblock \showarticletitle{Slave to the Algorithm: Why a Right to an
  Explanation Is Probably Not the Remedy You Are Looking for}.
\newblock \bibinfo{journal}{\emph{Duke L. \& Tech. Rev.}}  \bibinfo{volume}{16}
  (\bibinfo{year}{2017}), \bibinfo{pages}{18}.
\newblock


\bibitem[\protect\citeauthoryear{Fawzi, Fawzi, and Frossard}{Fawzi
  et~al\mbox{.}}{2018}]%
        {fawzi2018analysis}
\bibfield{author}{\bibinfo{person}{Alhussein Fawzi}, \bibinfo{person}{Omar
  Fawzi}, {and} \bibinfo{person}{Pascal Frossard}.}
  \bibinfo{year}{2018}\natexlab{}.
\newblock \showarticletitle{Analysis of classifiers' robustness to adversarial
  perturbations}.
\newblock \bibinfo{journal}{\emph{Machine Learning}} \bibinfo{volume}{107},
  \bibinfo{number}{3} (\bibinfo{year}{2018}), \bibinfo{pages}{481--508}.
\newblock


\bibitem[\protect\citeauthoryear{Hara, Ikeno, Soma, and Maehara}{Hara
  et~al\mbox{.}}{2018}]%
        {hara2018maximally}
\bibfield{author}{\bibinfo{person}{Satoshi Hara}, \bibinfo{person}{Kouichi
  Ikeno}, \bibinfo{person}{Tasuku Soma}, {and} \bibinfo{person}{Takanori
  Maehara}.} \bibinfo{year}{2018}\natexlab{}.
\newblock \showarticletitle{Maximally Invariant Data Perturbation as
  Explanation}.
\newblock \bibinfo{journal}{\emph{arXiv preprint arXiv:1806.07004}}
  (\bibinfo{year}{2018}).
\newblock


\bibitem[\protect\citeauthoryear{Hardt, Megiddo, Papadimitriou, and
  Wootters}{Hardt et~al\mbox{.}}{2016}]%
        {hardt2016strategic}
\bibfield{author}{\bibinfo{person}{Moritz Hardt}, \bibinfo{person}{Nimrod
  Megiddo}, \bibinfo{person}{Christos Papadimitriou}, {and}
  \bibinfo{person}{Mary Wootters}.} \bibinfo{year}{2016}\natexlab{}.
\newblock \showarticletitle{Strategic Classification}. In
  \bibinfo{booktitle}{\emph{Proceedings of the 2016 ACM Conference on
  Innovations in Theoretical Computer Science}}. ACM,
  \bibinfo{pages}{111--122}.
\newblock


\bibitem[\protect\citeauthoryear{Hu, Immorlica, and Vaughan}{Hu
  et~al\mbox{.}}{2019}]%
        {hu2019strategic}
\bibfield{author}{\bibinfo{person}{Lily Hu}, \bibinfo{person}{Nicole
  Immorlica}, {and} \bibinfo{person}{Jennifer~Wortman Vaughan}.}
  \bibinfo{year}{2019}\natexlab{}.
\newblock \showarticletitle{The Disparate Effects of Strategic Manipulation}.
  In \bibinfo{booktitle}{\emph{Proceedings of the Conference on Fairness,
  Accountability, and Transparency}} \emph{(\bibinfo{series}{FAT* '19})}.
  \bibinfo{publisher}{ACM}, \bibinfo{address}{New York, NY, USA},
  \bibinfo{pages}{259--268}.
\newblock


\bibitem[\protect\citeauthoryear{ILOG}{ILOG}{2018}]%
        {cplex}
\bibfield{author}{\bibinfo{person}{IBM ILOG}.} \bibinfo{year}{2018}\natexlab{}.
\newblock \bibinfo{title}{CPLEX Optimizer 12.8}.
\newblock
  \bibinfo{howpublished}{\url{https://www.ibm.com/analytics/cplex-optimizer}}.
\newblock


\bibitem[\protect\citeauthoryear{Kaggle}{Kaggle}{2011}]%
        {data2018givemecredit}
\bibfield{author}{\bibinfo{person}{Kaggle}.} \bibinfo{year}{2011}\natexlab{}.
\newblock \bibinfo{title}{{Give Me Some Credit}}.
\newblock
  \bibinfo{howpublished}{\url{http://www.kaggle.com/c/GiveMeSomeCredit/}}.
\newblock


\bibitem[\protect\citeauthoryear{Kallus and Zhou}{Kallus and Zhou}{2018}]%
        {kallus2018residual}
\bibfield{author}{\bibinfo{person}{Nathan Kallus} {and} \bibinfo{person}{Angela
  Zhou}.} \bibinfo{year}{2018}\natexlab{}.
\newblock \showarticletitle{Residual Unfairness in Fair Machine Learning from
  Prejudiced Data}. In \bibinfo{booktitle}{\emph{International Conference on
  Machine Learning}}.
\newblock


\bibitem[\protect\citeauthoryear{{Kleinberg} and {Raghavan}}{{Kleinberg} and
  {Raghavan}}{2018}]%
        {kleinberg2018strategic}
\bibfield{author}{\bibinfo{person}{Jon {Kleinberg}} {and}
  \bibinfo{person}{Manish {Raghavan}}.} \bibinfo{year}{2018}\natexlab{}.
\newblock \showarticletitle{{How Do Classifiers Induce Agents To Invest Effort
  Strategically?}}
\newblock \bibinfo{journal}{\emph{ArXiv e-prints}}, Article
  \bibinfo{articleno}{arXiv:1807.05307} (\bibinfo{date}{July}
  \bibinfo{year}{2018}), \bibinfo{numpages}{arXiv:1807.05307}~pages.
\newblock
\showeprint[arxiv]{cs.CY/1807.05307}


\bibitem[\protect\citeauthoryear{Lim and Dey}{Lim and Dey}{2009}]%
        {lim2009assessing}
\bibfield{author}{\bibinfo{person}{Brian~Y Lim} {and} \bibinfo{person}{Anind~K
  Dey}.} \bibinfo{year}{2009}\natexlab{}.
\newblock \showarticletitle{Assessing demand for intelligibility in
  context-aware applications}. In \bibinfo{booktitle}{\emph{Proceedings of the
  11th international conference on Ubiquitous computing}}. ACM,
  \bibinfo{pages}{195--204}.
\newblock


\bibitem[\protect\citeauthoryear{Malioutov and Varshney}{Malioutov and
  Varshney}{2013}]%
        {malioutov2013exact}
\bibfield{author}{\bibinfo{person}{Dmitry Malioutov} {and}
  \bibinfo{person}{Kush Varshney}.} \bibinfo{year}{2013}\natexlab{}.
\newblock \showarticletitle{Exact rule learning via boolean compressed
  sensing}. In \bibinfo{booktitle}{\emph{International Conference on Machine
  Learning}}. \bibinfo{pages}{765--773}.
\newblock


\bibitem[\protect\citeauthoryear{Martens and Provost}{Martens and
  Provost}{2014}]%
        {martens2014explaining}
\bibfield{author}{\bibinfo{person}{David Martens} {and} \bibinfo{person}{Foster
  Provost}.} \bibinfo{year}{2014}\natexlab{}.
\newblock \showarticletitle{Explaining data-driven document classifications}.
\newblock \bibinfo{journal}{\emph{MIS Quarterly}} \bibinfo{volume}{38},
  \bibinfo{number}{1} (\bibinfo{year}{2014}), \bibinfo{pages}{73--100}.
\newblock


\bibitem[\protect\citeauthoryear{Milli, Miller, Dragan, and Hardt}{Milli
  et~al\mbox{.}}{2019a}]%
        {milli2019strategic}
\bibfield{author}{\bibinfo{person}{Smitha Milli}, \bibinfo{person}{John
  Miller}, \bibinfo{person}{Anca~D. Dragan}, {and} \bibinfo{person}{Moritz
  Hardt}.} \bibinfo{year}{2019}\natexlab{a}.
\newblock \showarticletitle{The Social Cost of Strategic Classification}. In
  \bibinfo{booktitle}{\emph{Proceedings of the Conference on Fairness,
  Accountability, and Transparency}} \emph{(\bibinfo{series}{FAT* '19})}.
  \bibinfo{publisher}{ACM}, \bibinfo{address}{New York, NY, USA},
  \bibinfo{pages}{230--239}.
\newblock


\bibitem[\protect\citeauthoryear{Milli, Schmidt, Dragan, and Hardt}{Milli
  et~al\mbox{.}}{2019b}]%
        {milli2019model}
\bibfield{author}{\bibinfo{person}{Smitha Milli}, \bibinfo{person}{Ludwig
  Schmidt}, \bibinfo{person}{Anca~D Dragan}, {and} \bibinfo{person}{Moritz
  Hardt}.} \bibinfo{year}{2019}\natexlab{b}.
\newblock \showarticletitle{Model Reconstruction from Model Explanations}. In
  \bibinfo{booktitle}{\emph{Proceedings of the Conference on Fairness,
  Accountability, and Transparency}}. ACM, \bibinfo{pages}{1--9}.
\newblock


\bibitem[\protect\citeauthoryear{Mittleman}{Mittleman}{2018}]%
        {mittlemanmip2018}
\bibfield{author}{\bibinfo{person}{Hans Mittleman}.}
  \bibinfo{year}{2018}\natexlab{}.
\newblock \bibinfo{title}{Mixed Integer Linear Programming Benchmarks (MIPLIB
  2010)}.
\newblock \bibinfo{howpublished}{\url{http://plato.asu.edu/ftp/milpc.html}}.
\newblock


\bibitem[\protect\citeauthoryear{O'Neil}{O'Neil}{2016}]%
        {cathy2016weapons}
\bibfield{author}{\bibinfo{person}{Cathy O'Neil}.}
  \bibinfo{year}{2016}\natexlab{}.
\newblock \bibinfo{booktitle}{\emph{Weapons of math destruction: How big data
  increases inequality and threatens democracy}}.
\newblock \bibinfo{publisher}{Broadway Books}.
\newblock


\bibitem[\protect\citeauthoryear{Poulin, Eisner, Szafron, Lu, Greiner, Wishart,
  Fyshe, Pearcy, MacDonell, and Anvik}{Poulin et~al\mbox{.}}{2006}]%
        {poulin2006visual}
\bibfield{author}{\bibinfo{person}{Brett Poulin}, \bibinfo{person}{Roman
  Eisner}, \bibinfo{person}{Duane Szafron}, \bibinfo{person}{Paul Lu},
  \bibinfo{person}{Russell Greiner}, \bibinfo{person}{David~S Wishart},
  \bibinfo{person}{Alona Fyshe}, \bibinfo{person}{Brandon Pearcy},
  \bibinfo{person}{Cam MacDonell}, {and} \bibinfo{person}{John Anvik}.}
  \bibinfo{year}{2006}\natexlab{}.
\newblock \showarticletitle{Visual explanation of evidence with additive
  classifiers}. In \bibinfo{booktitle}{\emph{Proceedings Of The National
  Conference On Artificial Intelligence}}, Vol.~\bibinfo{volume}{21}. Menlo
  Park, CA; Cambridge, MA; London; AAAI Press; MIT Press; 1999,
  \bibinfo{pages}{1822}.
\newblock


\bibitem[\protect\citeauthoryear{Reisman, Schultz, Crawford, and
  Whittaker}{Reisman et~al\mbox{.}}{2018}]%
        {reisman2018aia}
\bibfield{author}{\bibinfo{person}{Dillon Reisman}, \bibinfo{person}{Jason
  Schultz}, \bibinfo{person}{Kate Crawford}, {and} \bibinfo{person}{Meredith
  Whittaker}.} \bibinfo{year}{2018}\natexlab{}.
\newblock \bibinfo{title}{Algorithmic Impact Assessments: A Practical Framework
  for Public Agency Accountability}.
\newblock \bibinfo{howpublished}{AI Now Technical Report}.
\newblock


\bibitem[\protect\citeauthoryear{Ribeiro, Singh, and Guestrin}{Ribeiro
  et~al\mbox{.}}{2016}]%
        {ribeiro2016should}
\bibfield{author}{\bibinfo{person}{Marco~Tulio Ribeiro},
  \bibinfo{person}{Sameer Singh}, {and} \bibinfo{person}{Carlos Guestrin}.}
  \bibinfo{year}{2016}\natexlab{}.
\newblock \showarticletitle{{Why should I trust you?: Explaining the
  predictions of any classifier}}. In \bibinfo{booktitle}{\emph{Proceedings of
  the 22nd ACM SIGKDD International Conference on Knowledge Discovery and Data
  Mining}}. ACM, \bibinfo{pages}{1135--1144}.
\newblock


\bibitem[\protect\citeauthoryear{Ribeiro, Singh, and Guestrin}{Ribeiro
  et~al\mbox{.}}{2018}]%
        {ribeiro2018anchors}
\bibfield{author}{\bibinfo{person}{Marco~Tulio Ribeiro},
  \bibinfo{person}{Sameer Singh}, {and} \bibinfo{person}{Carlos Guestrin}.}
  \bibinfo{year}{2018}\natexlab{}.
\newblock \showarticletitle{Anchors: High-precision model-agnostic
  explanations}. In \bibinfo{booktitle}{\emph{AAAI Conference on Artificial
  Intelligence}}.
\newblock


\bibitem[\protect\citeauthoryear{Scism}{Scism}{2019}]%
        {wsj2019lifeinsurance}
\bibfield{author}{\bibinfo{person}{Leslie Scism}.}
  \bibinfo{year}{2019}\natexlab{}.
\newblock \bibinfo{title}{New York Insurers Can Evaluate Your Social Media Use
  - If They Can Prove Why It's Needed}.
\newblock
\newblock


\bibitem[\protect\citeauthoryear{Shroff}{Shroff}{2017}]%
        {shroff2017predictive}
\bibfield{author}{\bibinfo{person}{Ravi Shroff}.}
  \bibinfo{year}{2017}\natexlab{}.
\newblock \showarticletitle{Predictive Analytics for City Agencies: Lessons
  from Children's Services}.
\newblock \bibinfo{journal}{\emph{Big data}} \bibinfo{volume}{5},
  \bibinfo{number}{3} (\bibinfo{year}{2017}), \bibinfo{pages}{189--196}.
\newblock


\bibitem[\protect\citeauthoryear{Siddiqi}{Siddiqi}{2012}]%
        {siddiqi2012credit}
\bibfield{author}{\bibinfo{person}{Naeem Siddiqi}.}
  \bibinfo{year}{2012}\natexlab{}.
\newblock \bibinfo{booktitle}{\emph{Credit Risk Scorecards: Developing and
  Implementing Intelligent Credit Scoring}}. Vol.~\bibinfo{volume}{3}.
\newblock \bibinfo{publisher}{John Wiley \& Sons}.
\newblock


\bibitem[\protect\citeauthoryear{Sokolovska, Chevaleyre, and Zucker}{Sokolovska
  et~al\mbox{.}}{2018}]%
        {sokolovska2018provable}
\bibfield{author}{\bibinfo{person}{Nataliya Sokolovska}, \bibinfo{person}{Yann
  Chevaleyre}, {and} \bibinfo{person}{Jean-Daniel Zucker}.}
  \bibinfo{year}{2018}\natexlab{}.
\newblock \showarticletitle{A Provable Algorithm for Learning Interpretable
  Scoring Systems}. In \bibinfo{booktitle}{\emph{International Conference on
  Artificial Intelligence and Statistics}}. \bibinfo{pages}{566--574}.
\newblock


\bibitem[\protect\citeauthoryear{Spangher and Ustun}{Spangher and
  Ustun}{2018}]%
        {recourse2018fatml}
\bibfield{author}{\bibinfo{person}{Alexander Spangher} {and}
  \bibinfo{person}{Berk Ustun}.} \bibinfo{year}{2018}\natexlab{}.
\newblock \showarticletitle{{Actionable Recourse in Linear Classification}}. In
  \bibinfo{booktitle}{\emph{Proceedings of the 5th Workshop on Fairness,
  Accountability and Transparency in Machine Learning}}.
\newblock


\bibitem[\protect\citeauthoryear{Taylor}{Taylor}{1980}]%
        {taylor1980meeting}
\bibfield{author}{\bibinfo{person}{Winnie~F Taylor}.}
  \bibinfo{year}{1980}\natexlab{}.
\newblock \showarticletitle{Meeting the Equal Credit Opportunity Act's
  Specificity Requirement: Judgmental and Statistical Scoring Systems}.
\newblock \bibinfo{journal}{\emph{Buff. L. Rev.}}  \bibinfo{volume}{29}
  (\bibinfo{year}{1980}), \bibinfo{pages}{73}.
\newblock


\bibitem[\protect\citeauthoryear{Tomlin}{Tomlin}{1988}]%
        {tomlin1988special}
\bibfield{author}{\bibinfo{person}{John~A Tomlin}.}
  \bibinfo{year}{1988}\natexlab{}.
\newblock \showarticletitle{Special ordered sets and an application to gas
  supply operations planning}.
\newblock \bibinfo{journal}{\emph{Mathematical programming}}
  \bibinfo{volume}{42}, \bibinfo{number}{1-3} (\bibinfo{year}{1988}),
  \bibinfo{pages}{69--84}.
\newblock


\bibitem[\protect\citeauthoryear{Tram{\`e}r, Zhang, Juels, Reiter, and
  Ristenpart}{Tram{\`e}r et~al\mbox{.}}{2016}]%
        {tramer2016stealing}
\bibfield{author}{\bibinfo{person}{Florian Tram{\`e}r}, \bibinfo{person}{Fan
  Zhang}, \bibinfo{person}{Ari Juels}, \bibinfo{person}{Michael~K Reiter},
  {and} \bibinfo{person}{Thomas Ristenpart}.} \bibinfo{year}{2016}\natexlab{}.
\newblock \showarticletitle{Stealing Machine Learning Models via Prediction
  APIs.}. In \bibinfo{booktitle}{\emph{USENIX Security Symposium}}.
  \bibinfo{pages}{601--618}.
\newblock


\bibitem[\protect\citeauthoryear{{United States Congress}}{{United States
  Congress}}{2003}]%
        {congress2003facta}
\bibfield{author}{\bibinfo{person}{{United States Congress}}.}
  \bibinfo{year}{2003}\natexlab{}.
\newblock \bibinfo{title}{The Fair and Accurate Credit Transactions Act}.
\newblock
\newblock


\bibitem[\protect\citeauthoryear{{United States Senate}}{{United States
  Senate}}{2019}]%
        {senate2019aaa}
\bibfield{author}{\bibinfo{person}{{United States Senate}}.}
  \bibinfo{year}{2019}\natexlab{}.
\newblock \bibinfo{title}{{Algorithmic Accountability Act of 2019}}.
\newblock
\newblock


\bibitem[\protect\citeauthoryear{Ustun and Rudin}{Ustun and Rudin}{2016}]%
        {ustun2016slim}
\bibfield{author}{\bibinfo{person}{Berk Ustun} {and} \bibinfo{person}{Cynthia
  Rudin}.} \bibinfo{year}{2016}\natexlab{}.
\newblock \showarticletitle{{Supersparse Linear Integer Models for Optimized
  Medical Scoring Systems}}.
\newblock \bibinfo{journal}{\emph{Machine Learning}} \bibinfo{volume}{102},
  \bibinfo{number}{3} (\bibinfo{year}{2016}), \bibinfo{pages}{349--391}.
\newblock


\bibitem[\protect\citeauthoryear{Ustun and Rudin}{Ustun and Rudin}{2017}]%
        {ustun2016kdd}
\bibfield{author}{\bibinfo{person}{Berk Ustun} {and} \bibinfo{person}{Cynthia
  Rudin}.} \bibinfo{year}{2017}\natexlab{}.
\newblock \showarticletitle{{Optimized Risk Scores}}. In
  \bibinfo{booktitle}{\emph{Proceedings of the 23rd ACM SIGKDD International
  Conference on Knowledge Discovery and Data Mining}}. ACM.
\newblock


\bibitem[\protect\citeauthoryear{Wachter, Mittelstadt, and Russell}{Wachter
  et~al\mbox{.}}{2017}]%
        {wachter2017counterfactual}
\bibfield{author}{\bibinfo{person}{Sandra Wachter}, \bibinfo{person}{Brent
  Mittelstadt}, {and} \bibinfo{person}{Chris Russell}.}
  \bibinfo{year}{2017}\natexlab{}.
\newblock \showarticletitle{Counterfactual Explanations without Opening the
  Black Box: Automated Decisions and the GDPR}.
\newblock  (\bibinfo{year}{2017}).
\newblock


\bibitem[\protect\citeauthoryear{Wilhelm}{Wilhelm}{2018}]%
        {politico2018creditgap}
\bibfield{author}{\bibinfo{person}{Colin Wilhelm}.}
  \bibinfo{year}{2018}\natexlab{}.
\newblock \bibinfo{title}{Big Data and the Credit Gap}.
\newblock
  \bibinfo{howpublished}{\url{https://www.politico.com/agenda/story/2018/02/07/big-data-credit-gap-000630}}.
\newblock


\bibitem[\protect\citeauthoryear{Yeh and Lien}{Yeh and Lien}{2009}]%
        {yeh2009comparisons}
\bibfield{author}{\bibinfo{person}{I-Cheng Yeh} {and} \bibinfo{person}{Che-hui
  Lien}.} \bibinfo{year}{2009}\natexlab{}.
\newblock \showarticletitle{The Comparisons of Data Mining Techniques for the
  Predictive Accuracy of Probability of Default of Credit Card Clients}.
\newblock \bibinfo{journal}{\emph{Expert Systems with Applications}}
  \bibinfo{volume}{36}, \bibinfo{number}{2} (\bibinfo{year}{2009}),
  \bibinfo{pages}{2473--2480}.
\newblock


\end{thebibliography}

\clearpage
\appendix
\section{Omitted Proofs}
\label{Appendix::Proofs}

\subsection*{Remark \ref{Rem::FeasibilitySufficient}}

\begin{proof}
Given a classifier $f: \X \to \{-1,+1\}$, let us define the space of feature vectors that are assigned a negative and positive label as $\hneg = \{\xb \in \X ~|~ f(\xb) = -1\}$ and $\hpos = \{\xb \in \X ~|~ f(\xb) = +1\}$, respectively. Since the classifier $f$ does not trivially predict a single class over the target population, there must exist at least one feature vector $\xb \in \hneg$ and at least one feature vector $\xb' \in \hpos$. 

Given any feature vector $\xb \in \hneg$, choose a fixed point $\xb' \in \hpos$. Since all features are actionable, the set of feasible actions from $\xb$ must contain an action vector $\a = \xb' - \xb.$ Thus, the classifier provides $\xb$ with recourse as $f(\xb +\a) = f(\xb + \xb' - \xb) = f(\xb') = +1$. Since our choice of $\xb$ was arbitrary, the previous result holds for all feature vectors $\xb \in \hneg$. Thus, the classifier provides recourse to all individuals in the target population.
\end{proof}

\subsection*{Remark \ref{Remark::UnboundedFeasibility}}
\begin{proof}
Given a linear classifier with coefficients $\w \in \R^{d+1}$, let $j$ denote the index of feature that can be increased or decreased arbitrarily. Assume, without loss of generality, that $w_j > 0$. Given a feature vector $\xb$ such that $f(\xb) = \sign{\w^{\top}\xb} = -1$, the set of feasible actions from $\xb$ must contain an action vector $\a = [0,a_1,a_2, \ldots, a_d]$ such that $a_j > - \frac{1}{w_j}\w^{\top} \xb$ and $a_{k} = 0$ for all $k \neq j$. Thus, the classifier provides $\xb$ with recourse as $\w^{\top}(\xb+\a) > 0$ and $f(\xb + \a) = \sign{\w^{\top}(\xb+\a)} =  +1$. Since our choice of $\xb$ was arbitrary, the result holds for all $\xb \in \hneg$. Thus, the classifier provides recourse to all individuals in the target population.
\end{proof}

\subsection*{Remark \ref{Remark::BoundedInfeasibility}}

\begin{proof}
Suppose we have $d$ actionable features $x_j \in \{0,1\}$ for $j \in \{1,\ldots,d\}$ and 1 immutable feature $x_{d+1} \in \{0,1\}$. Consider a linear classifier with the score function $\sum_{j=1}^{d} x_j + \alpha x_{j+1} \geq d$ where $\alpha < -1$. For any $\xb$ with $x_{d+1} = 1$, we have that $\sum_{j=1}^{d} x_j + \alpha x_{j+1} < \sum_{j=1}^{d} x_j - 1 \leq d - 1$. Thus, $\xb$ will not have recourse.
\end{proof}

\subsection*{Theorem \ref{Thm::ExpectedCost}}
\renewcommand{\P}[0]{{\mathbb{P}}}
\newcommand{\gx}[0]{{u_{\xb}}}

In what follows, we denote the unit score of actionable features from $\xb$ as $\gx{} = \frac{\wa^{\top}\xa}{\twonorm{\wa}^2}$. Our proof uses the following lemma from \citet{fawzi2018analysis}, which we have reproduced below for completeness:
\begin{lemma}[\citet{fawzi2018analysis}]
\label{lemma:bounded:cost}
Given a non-trivial linear classifier where $\wa \neq 0$, the optimal cost of recourse from $\xb \in \hneg$ obeys $$\cost{\a^*}{\xb} = \cx \cdot \frac{|\wa^\top\xa|}{\twonorm{\wa}^2} = \cx{} \gx{}.$$
\end{lemma}
\begin{proof}
Using the definition of $\expmincostneg{f}$, we can express:
\begin{align}
\expmincostneg{f} =&\quad p^+ \cdot \bigl(\E_{\hneg \cap \dpos }[\cost{\a^*}{\xb}|\gx{}\leq 0] \cdot \P_{\hneg \cap \dpos}(\gx{}\leq 0) + \label{term1}\\
& \quad\qquad \E_{\hneg \cap \dpos }[\cost{\a^*}{\xb}|\gx{}\geq 0] \cdot \P_{\hneg \cap \dpos}(\gx{}\geq 0)\bigr) \label{term2}\\
&+ p^- \cdot\bigl( \E_{\hneg \cap \dneg }[\cost{\a^*}{\xb}|\gx{}\leq 0] \cdot \P_{\hneg \cap \dneg}(\gx{}\leq 0) + \label{term3}\\
& \quad\qquad  \E_{\hneg \cap \dneg }[\cost{\a^*}{\xb}|\gx{}\geq 0] \cdot \P_{\hneg \cap \dneg}(\gx{}\geq 0)\bigr)\label{term4}
\end{align}
Using Lemma \ref{lemma:bounded:cost}, we can write the expectation term in line \eqref{term1} as:
\begin{align}
\E_{\hneg \cap \dpos}[\cost{\a^*}{\xb}|\gx{}\leq 0] = \E_{\hneg \cap \dpos }[-\cx{}\gx{}|\gx{} \leq 0] \label{term11}
\end{align}

\noindent%
Applying Lemma \ref{lemma:bounded:cost} to the expectation terms in lines \eqref{term2} to \eqref{term4}, we can write $\expmincostneg{f}$ as follows:
\begin{align}
\begin{split}
p^+\cdot \bigl(\E_{\hneg \cap \dpos }[-\cx{}\gx{}|\gx{}\leq 0] \cdot \P_{\hneg \cap \dpos}(\gx{}\leq 0) + \E_{\hneg \cap \dpos }[\cx{} \gx{}|\gx{}\geq 0] \cdot  \P_{\hneg \cap \dpos}(\gx{}\geq 0)\bigr) \\
+ p^- \cdot\bigl(\E_{\hneg \cap \dneg }[-\cx{}\gx{}|\gx{}\leq 0] \cdot  \P_{\hneg \cap \dneg}(\gx{}\leq 0) + \E_{\hneg \cap \dneg }[\cx{}\gx{}|\gx{}\geq 0] \cdot \P_{\hneg \cap \dneg}(\gx{}\geq 0)\bigr)
\end{split}
\label{midd}
\end{align}

\noindent%
We observe that the quantity in line \eqref{term11} can be bounded as follows:
\begin{align*}
&p^+ \cdot \E_{\hneg \cap \dpos }[-\cx{}\gx{}|\gx{}\leq 0] \cdot \P_{\hneg \cap \dpos}(\gx{}\leq 0)\\
=& 2p^+ \cdot |\E_{\hneg \cap \dpos}[-\cx{}\gx{})|\gx{}\leq 0]| \cdot \P_{\hneg \cap \dpos}(\gx{}\leq 0) +  p^+ \cdot \E_{\hneg \cap \dpos}[\cx{}\gx{}|\gx{}\leq 0] \cdot \P_{\hneg \cap \dpos}(\gx{}\leq 0)\\
\leq & 2p^+ \P_{\hneg \cap \dpos}(\gx{} \leq 0) \cdot \gamma^{\max}_A  + p^+ \E_{\hneg \cap \dpos}[\cx{}\gx{}|\gx{}\leq 0] \cdot \P_{\hneg \cap \dpos}(\gx{}\leq 0).
\end{align*}
Here, the inequality follows from the definition of $\gamma^{\max}_A$. 

\noindent%
We observe that the quantity in line \eqref{term3} can also be bounded in a similar manner: 
\begin{align*}
&p^- \cdot \E_{\hneg \cap \dneg }[\cx{}\gx{}|\gx{}\geq 0] \cdot \P_{\hneg \cap \dneg}(\gx{}\geq 0)\\
&\leq 2p^- \P_{\hneg \cap \dneg}(\gx{}\geq 0) \cdot \gamma^{\max}_A + p^- \E_{\hneg \cap \dneg}[\cx{}(-\gx{})|\gx{}\leq 0] \cdot \P_{\hneg \cap \dneg}(\gx{}\geq 0).
\end{align*}
Combining these inequalities with Equation \eqref{midd}, we obtain:
\begin{align*}
\expmincostneg{f} \leq &\; p^+ \bigl( 2\P_{\hneg \cap \dpos}(\gx{}\leq 0) \cdot \ucmax{} \\
& \;\;\quad + \E_{\hneg \cap \dpos }[\cx{}\gx{}|\gx{}\leq 0] \cdot \P_{\hneg \cap \dpos}(\gx{}\leq 0)\\
& \;\;\quad +\E_{\hneg \cap \dpos }[\cx{}\gx{}|\gx{}\geq 0] \cdot \P_{\hneg \cap \dpos}(\gx{}\geq 0)\bigr)\\
&+ p^- \bigl(2\P_{\hneg \cap \dneg}(g(x)\geq 0)\cdot  \ucmax{}\\
& \;\;\quad +\E_{\hneg \cap \dneg }[-\cx{}\gx{})|\gx{}\leq 0] \cdot \P_{\hneg \cap \dneg}(\gx{}\leq 0)\\
& \;\;\quad +\E_{\hneg \cap \dneg }[-\cx{}\gx{})|\gx{}\geq 0] \cdot \P_{\hneg \cap \dneg}(\gx{}\geq 0) \bigr)\\
= &\; p^+ \bigl(\E_{\hneg \cap \dpos }[\cx{}\gx{}]+2\P_{\hneg \cap \dpos}(\gx{}\leq 0) \cdot \ucmax{} \bigr) \\
&+ p^-\bigl( \E_{\hneg \cap \dneg }[-\cx{}\gx{})]+2\P_{\hneg \cap \dpos}(\gx{}\leq 0) \cdot \ucmax{} \bigr)\\
= &\; p^+ \ucpos{} + p^- \ucneg{} + 2\ucmax R_A(f)
\end{align*}
\end{proof}

\clearpage
\section{Discretization Guarantees}
\label{Appendix::Discretization}

\newcommand{\amin}[1]{a_{#1}^\textnormal{min}}
\newcommand{\amax}[1]{a_{#1}^\textnormal{max}}
\newcommand{\Jcts}{J_\textnormal{cts}}
\newcommand{\Jdisc}{J_\textnormal{disc}}

In Section \ref{Sec::Methodology}, we state that discretization will not affect the feasibility or cost of recourse if we choose a suitable grid. In what follows, we present formal guarantees for this claim. Specifically, we show that:
\begin{enumerate}
    
    \item Discretization does not affect the feasibility of recourse if the actions for real-valued features are discretized onto a grid with the same upper and lower bounds.
    
    \item The maximum discretization error in the cost of recourse can be bounded and controlled by refining the grid.
    
\end{enumerate}

\subsection{Feasibility Guarantee}

\begin{proposition}
Given a linear classifier with coefficients $\w$, consider determining the feasibility of recourse for a person with features $\xb \in \X$ where the set of actions for each feature $j$ belong to a bounded interval $\A_j(\xb) = [\amin{j},\amax{j}] \subset \R$. Say we solve an instance of the integer program (IP) \eqref{IP::RecourseIP} using a discretized action set $\A^\textrm{disc}(\xb).$ If $\A^\textrm{disc}_j(\xb)$ contains the end points of $\A_j(\xb)$ for each $j$, then the IP will be infeasible whenever the person does not have recourse.
\end{proposition}

\begin{proof}
When the set of actions for each feature belong to a bounded interval $\A_j(\xb) = [\amin{j},\amax{j}]$, we have that:
\begin{align*}
\max_{\a \in \A(\xb)} f(\xb + \a) &= \max_{\a \in \A(\xb)} \w^\top(\xb + \a) \\
&= \w^\top\xb + \max_{\a \in \A(\xb)} \w^\top\a \\
&= \w^\top\xb + \sum_{j \in \Ja} \max_{a_j \in \A_j(\xb)} w_j a_j \\
&= \w^\top\xb + \sum_{j \in \Ja: w_j < 0} w_j \amin{j} + \sum_{j: w_j > 0} w_j \amax{j}\\
&= \max_{\a \in \A^\textrm{disc}(\xb)} f(\xb + \a)
\end{align*}
Thus, we have shown that: $$\max_{\a \in \A(\xb)} f(\xb + \a) = \max_{\a \in \A^\textrm{disc}(\xb)} f(\xb + \a).$$  
Observe that IP \eqref{IP::RecourseIP} is infeasible whenever $\max_{\a \in \A^\textrm{disc}(\xb)} f(\xb + \a) < 0,$ because this would violate constraint \eqref{Con::IPThreshold}. Since $$\max_{\a \in \A^\textrm{disc}(\xb)} f(\xb + \a) < 0 \Leftrightarrow \max_{\a \in \A(\xb)} f(\xb + \a) < 0,$$ it follows that the IP is infeasible whenever the person has no recourse under the original action set.
\end{proof}

\subsection{Cost Guarantee}

We present a bound on the maximum error in the cost of recourse due to the discretization of the cost function $\cost{\a}{\xb} = \cx{} \cdot \vnorm{\a}$, where $c: \X \to (0, +\infty)$ is a strictly positive scaling function for actions from $\xb$. Given a feature vector $\xb \in \X$, we denote the discretized action set from $\xb$ as $A(\xb)$ and the continuous action set as $B(\xb)$. We denote the minimal-cost action over $A(\xb)$ as:
\begin{align}
\begin{split}
\a^* \in \argmin \quad&  \cx{} \cdot ||\a|| \\ 
\st\quad & f(\xb + \a) = 1\\ 
& \a \in \A(\xb)
\end{split}
\end{align}
and the minimal-cost action over $B(\xb)$ as:
\begin{align}
\begin{split}
\b^* \in \argmin\quad&  \cx{} \cdot ||\b|| \\ 
\st\quad & f(\xb + \b) = 1\\ 
& \b \in B(\xb)
\end{split}
\end{align}
We assume that $\A(\xb) \subseteq \{\a \in \R^d ~|~ a_j \in \A_j(x_j) \}$ and denote the feasible actions for feature $j$ as $A_j(x_j) = \{0, a_{j1},...,a_{j,m_j}\}.$ We measure the refinement of the discrete grid in terms of the \emph{maximum discretization gap}: $$\sqrt{\sum_{j\in\J} \delta_j^2}.$$ Here $\delta_j = \max_{k=0,..,m_{j}-1} |a_{j,{k+1}} - a_{j,k}|.$ In Proposition \ref{Prop::DiscretizationCost}, we show that the difference in the cost of recourse due to discretization can be bounded in terms of the maximum discretization gap.
\begin{proposition}
\label{Prop::DiscretizationCost}
Given a linear classifier with coefficients $\w$, consider evaluating the cost of recourse for an individual with features $\xb \in \X.$ If the features belong to a bounded space $\xb \in \X$, then the cost can be bounded as:
\begin{align*}
    \cx{} \cdot ||\a^*|| - \cx{} \cdot ||\b^*||
    \leq \cx{} \cdot \sqrt{\sum_{j=1}^d \delta^2_j}.
\end{align*}
\end{proposition}
\begin{proof}

Let $N(\b^*)$ be a neighborhood of real-valued vectors centered at $\xb+\b^*$ and with radius $\sum_{j=1}^d \delta_{j}$:
$$N(\b^*) = \left\{ \xb': \vnorm{\xb'-(\xb+\b^*)} \leq \sqrt{\sum_{j=1}^d \delta^2_{j}} \right\}.$$
Observe that $N(\b^*)$ must contain an action $\hat{\a} \in A(\xb)$ such that $f(\xb+\hat{\a}) = +1$. By the triangle inequality, we can see that:
\begin{align}
\vnorm{\hat{\a}} &\leq  \vnorm{\b^*} + \vnorm{\hat{\a}-\b^*}, \notag \\
&\leq \vnorm{\b^*} +\sqrt{\sum_{j=1}^d \delta^2_{j}}. \label{Eq::D1}
\end{align}
Here the inequality in \eqref{Eq::D1} follows from the fact that $\xb+\hat{\a} \in N(\b^*)$. Since $\a^*$ is optimal, we have that $\vnorm{\a^*} \leq \vnorm{\hat{\a}}$. Thus we have,
\begin{align*}
\cx{} \cdot ||\a^*|| - \cx{} \cdot \vnorm{\b^*} &\leq \cx{} \cdot ||\hat{\a}|| - \cx{} \cdot \vnorm{\b^*}\\
    &\leq \cx{} \cdot \sqrt{\sum_{j=1}^d \delta^2_{j}}.
\end{align*}
\end{proof}

\subsection{IP Formulation without Discretization}
\label{Appendix::ContinuousIP}

In what follows, we present an IP formulation that does not require discretizing real-valued features, which we mention in Section \ref{Sec::Methodology}. Given a linear classifier with coefficients $\w = [w_0,\ldots, w_d]$ and a person  with features $\xb = [1,x_1,\ldots, x_d]$, we can recover the solution to the optimization problem in \eqref{Eq::RecourseProblem} for a linear cost function with the form $\cost{\a}{\xb} = \sum_j c_j a_j,$ by solving the IP:
\begin{subequations}
\small
\label{IP::RecourseIPCTS}
\begin{equationarray}{@{}r@{\;\;}l>{\;}l>{\;}r@{}} 
\min & \textrm{cost} \notag \\
\st{} & \textrm{cost} = \sum_{j \in \J} c_j a_j & \label{Con::CostCTS} \\ 
& \sum_{j \in \Ja} w_j a_j \geq - \sum_{j=0}^d {w_j x_j}  & \label{Con:FlipCTS} \\ 

& a_j  \in  [\amin{j}, \amax{j}]  &  j \in \Jcts \label{Con::ActionValueCTS} \\
& a_j  =  \sum_{k=1}^{m_j} a_{jk} v_{jk} & j \in \Jdisc & \notag \\ 
& 1 = u_j + \sum_{k=1}^{m_j} v_{jk}  &   j \in \Jdisc  & \label{Con::ActionValueDisc} \\
& a_j  \in  \R &  j \in \Jdisc  \notag \\ 
& u_{j} \in \B &  j \in \Jdisc \notag \\ 
& v_{jk} \in \B &  \miprange{k}{1}{m_j}\;j \in \Jdisc \notag
\end{equationarray}
\end{subequations}
Here, we denote the indices of actionable features as $\Ja = \Jcts \cup \Jdisc$, where $\Jcts$ and $\Jdisc$ correspond to the indices of real-valued features and discrete-valued features, respectively. This formulation differs from the discretized formulation in Section \ref{Sec::Methodology} in that: (i) it represents actions for real-valued features via continuous variables $a_j \in [\amin{j}, \amax{j}]$ in \eqref{Con::ActionValueCTS}; (ii) it only includes indicator variables $u_{j}$ and $v_{jk}$ and constraints \eqref{Con::ActionValueDisc} for discrete-valued variables $j \in \Jdisc$; 

The formulation in \eqref{IP::RecourseIPCTS} has the following drawbacks:
\begin{enumerate}

\item It forces users to use linear cost functions. This significantly restricts the ability of users to specify useful cost functions. Examples include cost functions based on percentile shifts, such as those in \eqref{Eq::CostFunctionAudit} and \eqref{Eq::CostFunctionFlipsets}, which are non-convex.
    
\item It is harder to optimize when we introduce constraints on feasible actions. If we wish to limit the number of features that can be altered in IP \eqref{IP::RecourseIPCTS}, for example, we must add indicator variables of the form $u_{j} = 1[a_j \neq 0]$ for real-valued features $j \in \Jcts$. These variables must be set via ``Big-M" constraints, which produce weak LP relaxations and numerical instability \citep[see e.g.,][for a discussion]{belotti2016handling}.
    
\end{enumerate}

\newpage
\section{Supporting Material for Section \ref{Sec::Demonstration}}
\label{Appendix::Experiments}

\subsection{Supporting Material for Section \ref{Sec::Demo1}}
\label{Appendix::Demo1}

\begin{table}[htbp]
\centering
\resizebox{0.7\textwidth}{!}{
\begin{tabular}{>{\itshape}lccccc}
\toprule
\normalfont{\textbf{Feature}} &          
\normalfont{\textbf{Type}} &   
\normalfont{\textbf{LB}}&  
\normalfont{\textbf{UB}}&  
\normalfont{\textbf{\# Actions}} &  
\normalfont{\textbf{Mutable}} \\
\midrule
Married &     $\{0,1\}$ &  0 &      1 &           2 & N \\
Single &     $\{0,1\}$ &  0 &      1 &           2 &  N\\
Age $<25$ &     $\{0,1\}$ &  0 &      1 &           2 & N\\
Age $\in [25,39]$ &     $\{0,1\}$ &  0 &      1 &           2 &N\\
Age $\in [40,59]$ &     $\{0,1\}$ &  0 &      1 &           2 &   N\\
Age $\geq$ 60 &     $\{0,1\}$ &  0 &      1 &           2 &    N\\
EducationLevel &  $\mathbb{Z}$ &  0 &      3 &           4 &  Y\\
MaxBillAmountOverLast6Months &  $\mathbb{Z}$ &  0 &  17091 &        3420 &      Y\\
MaxPaymentAmountOverLast6Months &  $\mathbb{Z}$ &  0 &  11511 &        2304 &    Y\\
MonthsWithZeroBalanceOverLast6Months &  $\mathbb{Z}$ &  0 &      6 &           7 & Y\\
MonthsWithLowSpendingOverLast6Months &  $\mathbb{Z}$ &  0 &      6 &           7 & Y\\
MonthsWithHighSpendingOverLast6Months &  $\mathbb{Z}$ &  0 &      6 &           7 & Y\\
MostRecentBillAmount &  $\mathbb{Z}$ &  0 &  15871 &        3176 & Y\\
MostRecentPaymentAmount &  $\mathbb{Z}$ &  0 &   7081 &        1418 & Y\\
TotalOverdueCounts &  $\mathbb{Z}$ &  0 &      2 &           3 & N\\
TotalMonthsOverdue &  $\mathbb{Z}$ &  0 &     32 &          33 & N\\
HistoryOfOverduePayments &     $\{0,1\}$ &  0 &      1 &  2 & N\\
\bottomrule
\end{tabular}
}
\caption{Overview of features and action set for \textds{credit}.}
\end{table}

\begin{figure}[htbp]
\centering
\includegraphics[width=0.5\textwidth]{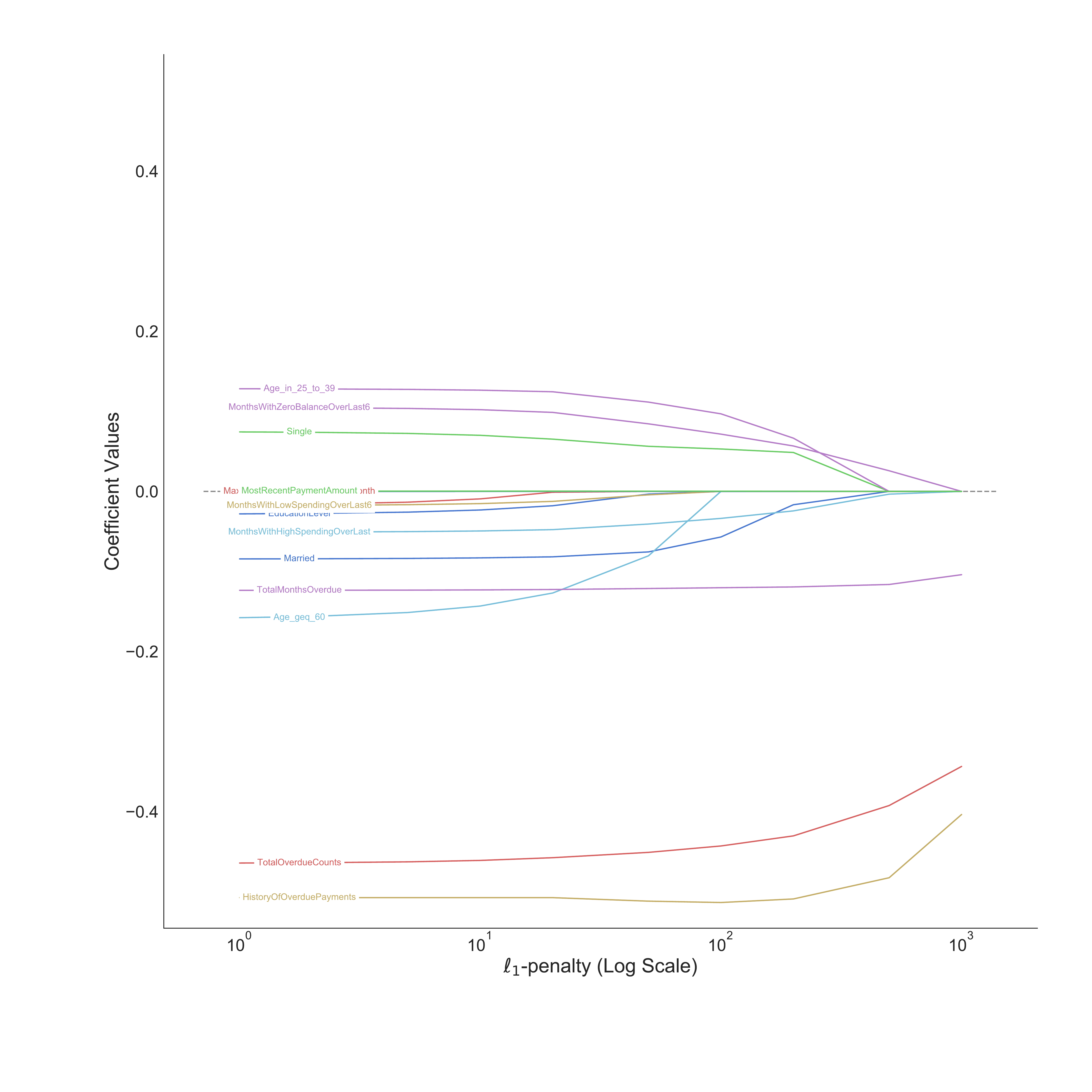}
\caption{Coefficients of $\ell_1$-penalized logistic regression models over the full $\ell_1$-regularization path for \textds{credit}.}
\end{figure}

\newpage
\subsection{Supporting Material for Section \ref{Sec::Demo2}}
\label{Appendix::Demo2}

\begin{table}[htbp]
\centering
\resizebox{0.7\linewidth}{!}{
\begin{tabular}{>{\itshape}lccccc}
\toprule
\normalfont{\textbf{Feature}} &          
\normalfont{\textbf{Type}} &   
\normalfont{\textbf{LB}}& 
\normalfont{\textbf{UB}}&  
\normalfont{\textbf{\# Actions}} &  
\normalfont{\textbf{Mutable}} \\
\midrule
 RevolvingUtilizationOfUnsecuredLines &  $\mathbb{R}$ &  0.0 &     1.1 &         101 &              Y \\
                                  Age &  $\mathbb{Z}$ &   24 &      87 &          64 &           N \\
 NumberOfTimes30-59DaysPastDueNotWorse &  $\mathbb{Z}$ &    0 &       4 &           5 &             Y \\
                            DebtRatio &  $\mathbb{R}$ &  0.0 &  5003.0 &         101 &              Y \\
                        MonthlyIncome &  $\mathbb{Z}$ &    0 &   23000 &         101 &              Y \\
      NumberOfOpenCreditLinesAndLoans &  $\mathbb{Z}$ &    0 &      24 &          25 &             Y \\
              NumberOfTimes90DaysLate &  $\mathbb{Z}$ &    0 &       3 &           4 &               Y \\
         NumberRealEstateLoansOrLines &  $\mathbb{Z}$ &    0 &       5 &           6 &              Y \\
 NumberOfTimes60-89DaysPastDueNotWorse &  $\mathbb{Z}$ &    0 &       2 &           3 &            Y \\
                   NumberOfDependents &  $\mathbb{Z}$ &    0 &       4 &           5 &             N \\
\bottomrule
\end{tabular}
}
\caption{Overview of features and actions for \textds{givemecredit}.}
\end{table}
\begin{table}[htbp]
\label{Table::GiveMeCreditModelBaseline}
\centering
\resizebox{0.45\linewidth}{!}{
\begin{tabular}{>{\itshape}lr}
\toprule
\normalfont{\textbf{Feature}} &  \textbf{Coefficient}  \\
\midrule
RevolvingUtilizationOfUnsecuredLines &      0.000060 \\
Age                                  &      0.038216 \\
NumberOfTime30-59DaysPastDueNotWorse &     -0.508338 \\
DebtRatio                            &      0.000070 \\
MonthlyIncome                        &      0.000036 \\
NumberOfOpenCreditLinesAndLoans      &      0.010703 \\
NumberOfTimes90DaysLate              &     -0.374242 \\
NumberRealEstateLoansOrLines         &     -0.065595 \\
NumberOfTime60-89DaysPastDueNotWorse &      0.852129 \\
NumberOfDependents                   &     -0.067961 \\
\bottomrule
\end{tabular}
}
\caption{Coefficients of the baseline $\ell_2$-penalized logistic regression model for \textds{givemecredit}. This classifier is trained on a representative sample from the target population. It has a mean 10-CV AUC of 0.693 and a training AUC of 0.698.}
\end{table}
\begin{table}[htbp]
\label{Table::GiveMeCreditModelBiased}
\centering
\resizebox{0.45\linewidth}{!}{
\begin{tabular}{>{\itshape}lr}
\toprule
\normalfont{\textbf{Feature}} &  \textbf{Coefficient} \\
\midrule
RevolvingUtilizationOfUnsecuredLines &      0.000084  \\
Age                                  &      0.048613  \\
NumberOfTime30-59DaysPastDueNotWorse &     -0.341624  \\
DebtRatio                            &      0.000085  \\
MonthlyIncome                        &      0.000036  \\
NumberOfOpenCreditLinesAndLoans      &      0.005317  \\
NumberOfTimes90DaysLate              &     -0.220071  \\
NumberRealEstateLoansOrLines         &     -0.037376 \\
NumberOfTime60-89DaysPastDueNotWorse &      0.337424  \\
NumberOfDependents                   &     -0.008662 \\
\bottomrule
\end{tabular}
}
\caption{Coefficients of the biased $\ell_2$-penalized logistic regression model for \textds{givemecredit}. This classifier is trained on the biased sample that undersamples young adults in the target population. It has a mean 10-CV test AUC of 0.710 and a training AUC of 0.725.}
\end{table}

\newpage
\subsection{Supporting Material for Section \ref{Sec::Demo3}}
\label{Appendix::Demo3}

\begin{table}[htbp]
\resizebox{0.6\linewidth}{!}{
\begin{tabular}{>{\itshape}lccccc}
\toprule
\normalfont{\textbf{Feature}} &          
\normalfont{\textbf{Type}} &   
\normalfont{\textbf{LB}}&  
\normalfont{\textbf{UB}}&  
\normalfont{\textbf{\# Actions}} &  
\normalfont{\textbf{Mutable}} \\
\midrule
   ForeignWorker &     $\{0,1\}$ &    0 &      1 &           2 &        N \\
   Single &     $\{0,1\}$ &    0 &      1 &           2 &          N \\
   Age &  $\mathbb{Z}$ &   20 &     67 &          48 &           N \\
   LoanDuration &  $\mathbb{Z}$ &    6 &     60 &          55 &  Y \\
   LoanAmount &  $\mathbb{Z}$ &  368 &  14318 &         101 &    Y \\
   LoanRateAsPercentOfIncome &  $\mathbb{Z}$ &    1 &      4 &          4 &     Y \\
   YearsAtCurrentHome &  $\mathbb{Z}$ &    1 &      4 &           4 &          Y \\
   NumberOfOtherLoansAtBank &  $\mathbb{Z}$ &    1 &      3 &           3 &     Y \\
   NumberOfLiableIndividuals &  $\mathbb{Z}$ &    1 &      2 &           2 &     Y \\
   HasTelephone &     $\{0,1\}$ &    0 &      1 &           2 &         Y \\
   CheckingAccountBalance $\geq$ 0 &     $\{0,1\}$ &    0 &      1 &           2 &        Y \\
   CheckingAccountBalance $\geq$ 200 &     $\{0,1\}$ &    0 &      1 &           2 &      Y \\
   SavingsAccountBalance $\geq$ 100 &     $\{0,1\}$ &    0 &      1 &           2 &       Y \\
   SavingsAccountBalance $\geq$ 500 &     $\{0,1\}$ &    0 &      1 &           2 &       Y \\
   MissedPayments &     $\{0,1\}$ &    0 &      1 &           2 &          Y \\
   NoCurrentLoan &     $\{0,1\}$ &    0 &      1 &           2 &     Y \\
   CriticalAccountOrLoansElsewhere &     $\{0,1\}$ &    0 &      1 &           2 &       Y \\
   OtherLoansAtBank &     $\{0,1\}$ &    0 &      1 &           2 &       Y \\
   HasCoapplicant &     $\{0,1\}$ &    0 &      1 &           2 &   Y \\
   HasGuarantor &     $\{0,1\}$ &    0 &      1 &           2 &     Y \\
   OwnsHouse &     $\{0,1\}$ &    0 &      1 &           2 &               N \\
   RentsHouse &     $\{0,1\}$ &    0 &      1 &           2 &           N \\
   Unemployed &     $\{0,1\}$ &    0 &      1 &           2 &            Y \\
   YearsAtCurrentJob $\leq$ 1 &     $\{0,1\}$ &    0 &      1 &           2 &       Y \\
   YearsAtCurrentJob $\geq$ 4 &     $\{0,1\}$ &    0 &      1 &           2 &       Y \\
   JobClassIsSkilled &     $\{0,1\}$ &    0 &      1 &           2 &            N \\
\bottomrule
\end{tabular}
}
\caption{Overview of features and actions for \textds{german}.}
\end{table}
%
%
\begin{table}[htbp]
\label{Table::GermanModel}
\centering
\resizebox{0.35\linewidth}{!}{
\begin{tabular}{>{\itshape}lr}
\toprule
\normalfont{\textbf{Feature}} &  \textbf{Coefficient} \\
\midrule
ForeignWorker                   &  0.327309 \\
Single                          &  0.389049 \\
Age                             &  0.016774 \\
LoanDuration                    & -0.025132 \\
LoanAmount                      & -0.000077 \\
LoanRateAsPercentOfIncome       & -0.238608 \\
YearsAtCurrentHome              &  0.051728 \\
NumberOfOtherLoansAtBank        & -0.259529 \\
NumberOfLiableIndividuals       &  0.024364 \\
HasTelephone                    &  0.403947 \\
CheckingAccountBalance $\geq$ 0    & -0.324129 \\
CheckingAccountBalance $\geq$ 200  &  0.253868 \\
SavingsAccountBalance $\geq$ 100   &  0.436276 \\
SavingsAccountBalance $\geq$ 500   &  0.516691 \\
MissedPayments                  &  0.219252 \\
NoCurrentLoan                   & -0.583011 \\
CriticalAccountOrLoansElsewhere &  0.786617 \\
OtherLoansAtBank                & -0.623621 \\
HasCoapplicant                  & -0.240802 \\
HasGuarantor                    &  0.369806 \\
OwnsHouse                       &  0.690955 \\
RentsHouse                      &  0.131176 \\
Unemployed                      & -0.172313 \\
YearsAtCurrentJob $\leq$ 1          & -0.201463 \\
YearsAtCurrentJob $\geq$ 4         &  0.416902 \\
JobClassIsSkilled               &  0.236540 \\\bottomrule
\end{tabular}
}
\caption{Coefficients of a $\ell_2$-penalized logistic regression model for \textds{german}. This model has a mean 10-CV test AUC of 0.713 and a training AUC of 0.749}
\end{table}

\end{document}